\documentclass{article}

\PassOptionsToPackage{numbers, compress}{natbib}

\usepackage[preprint]{neurips_2021}

\usepackage[utf8]{inputenc} 
\usepackage{natbib}
\setcitestyle{numbers}
\usepackage[T1]{fontenc}    
\usepackage{hyperref}       
\usepackage{url}            
\usepackage{booktabs}       
\usepackage{amsfonts}       
\usepackage{nicefrac}       
\usepackage{microtype}      
\usepackage{xcolor}         
\usepackage{amsmath, amsthm, amsfonts, amssymb}

\usepackage{siunitx}
\usepackage{graphicx}    
\usepackage{caption}     
\usepackage{subcaption}  
\usepackage{multirow, makecell}  

\newtheorem{prop}{Proposition}
\usepackage{algorithm}
\usepackage{algpseudocode}

\newcommand{\minus}{\scalebox{0.5}[1.0]{$-$}}
\title{Robust Out-of-Distribution Detection on
\\Deep Probabilistic Generative Models}

\author{
Jaemoo Choi\thanks{Equal Contribution}\\
\And 
Changyeon Yoon\footnotemark[1]\\
\And 
Jeongwoo Bae\\
\And 
Myungjoo Kang
\AND
\textnormal{Seoul National University}\\
}

\begin{document}

\maketitle

\begin{abstract}
Out-of-distribution (OOD) detection is an important task in machine learning systems for ensuring their reliability and safety. Deep probabilistic generative models facilitate OOD detection by estimating the likelihood of a data sample. However, such models frequently assign a suspiciously high likelihood to a specific outlier. Several recent works have addressed this issue by training a neural network with auxiliary outliers, which are generated by perturbing the input data. In this paper, we discover that these approaches fail for certain OOD datasets. Thus, we suggest a new detection metric that operates without outlier exposure. We observe that our metric is robust to diverse variations of an image compared to the previous outlier-exposing methods. Furthermore, our proposed score requires neither auxiliary models nor additional training. Instead, this paper utilizes the likelihood ratio statistic in a new perspective to extract genuine properties from the given single deep probabilistic generative model. We also apply a novel numerical approximation to enable fast implementation. Finally, we demonstrate comprehensive experiments on various probabilistic generative models and show that our method achieves state-of-the-art performance.
\end{abstract}

\section{Introduction}
Recognizing whether data are anomalous or significantly different from training data is a crucial task in the machine learning system for ensuring their reliability and safety.
Such anomalous data are known as out-of-distribution (OOD) data.
Unfortunately, deep neural network classifiers frequently classify OOD data erroneously into one of the categories in high confidence, which may cause fatal consequences when applied to real-world data \citep{nguyen2015deep}. For instance, misclassifying OOD data in autonomous driving or medical diagnosis fields are hazardous to human life.
Several approaches deal with this problem based on deep neural network classifiers but are limited only to labeled data.
In contrast to classifiers, deep generative models can not only function without supervision but are also widely applied in various domains such as images, graphs, and sequences \citep{kipf2016variational, malhotra2016lstm}.

Deep probabilistic generative models, especially those offering explicit computation of the marginal likelihood, such as auto-regressive models \citep{salimans2017pixelcnn++}, variational auto-encoders (VAEs) \citep{kingma2013auto}, and normalizing flows \citep{dinh2016density, kingma2018glow} are considered to be successful in OOD detection since they approximate the likelihoods of the input data and successfully generate realistic in-distribution samples.
However, the validity of generative models for OOD detection is challenged due to several detrimental phenomena \citep{nalisnick2018deep, choi2018waic}.
Numerous studies \citep{nalisnick2018deep, choi2018waic, kirichenko2020normalizing, pope2020adversarial, xiao2020likelihood} report difficulties in OOD detection on VAE \citep{kingma2013auto} and normalizing flows \citep{kingma2018glow}.
There have been various attempts to develop a universal score for deep probabilistic generative models; however, no unified method has been established, to the best of our knowledge.

Several studies \citep{choi2019novelty, ren2019likelihood, ran2020bigeminal, lee2017training, hendrycks2018deep} suggest discriminating the OOD samples by exposing outliers during training.
In our work, we call this methodology `outlier exposure' and classify those outliers into two main types: explicit outlier data and synthetic data generated by perturbation of input images.
This concept is a generalization of the outlier exposure mentioned in \citep{hendrycks2018deep}.
The assumption of the outlier exposure is that if a model gains the ability to distinguish in-distribution data from the OOD samples that are exposed during training, the model may well perform for arbitrary OOD samples.
However, since such outliers do not represent the entire OOD data in an image manifold, we show in Section \ref{section 3.1} that such an assumption is questionable.
Based on these observations, we propose a new evaluation metric for deep probabilistic generative models without using any prior knowledge of OOD data.

Our proposed metric, named Robust Out-of-distribution detection ScorE (ROSE), estimates the extent to which the trained model is expected to be improved under the assumption that a test sample is fed as an input.
Unlike most previous works \citep{xiao2020likelihood, choi2019novelty, ren2019likelihood}, our score is simply obtained by only backpropagating once and neither requires any outlier to be exposed nor additional training during the test phase.
Furthermore, ROSE is applicable to a single deep generative model and has a short inference time, thus cost-effective and practical.
We conduct ROSE on various datasets \citep{mnist, fmnist, cifar10, svhn, celeba, lsun, omniglot, kmnist} for two representative probabilistic generative models, VAE and generative flow (GLOW) \citep{kingma2018glow}, which are known to be egregious for the OOD detection task.
In Section \ref{section 5.1}, we present comprehensive experiments to show that our work achieves state-of-the-art performance compared to the other benchmark methods \citep{xiao2020likelihood, ren2019likelihood, serra2019input}.
In Section \ref{section 5.2}, we further demonstrate the robustness of our proposed metric in various ways.

Our contributions are summarized as the followings:
\begin{itemize}
    \item We show that OOD detectors relying on the outlier exposure with synthetic data are unstable.
    \item We propose a new evaluation metric for OOD detection, named ROSE, which is widely applicable for various probabilistic generative models such as VAEs and normalizing flows.
    \item We evaluate the existing OOD methods on benchmark image datasets and demonstrate that our method achieves state-of-the-art (SOTA) performance even without exposing outliers.
\end{itemize}

\section{Related Works} \label{section 2}
Several studies utilize prior knowledge of the OOD samples in the training procedure to learn information drawn from an additional OOD dataset \citep{lee2017training, hendrycks2018deep, sehwag2021ssd, ran2020bigeminal}.
Exposing realistic and diverse anomalies are known to enhance the overall OOD detection performance \citep{hendrycks2018deep}.
Numerous researches verify its effectiveness in various deep learning architectures such as deep classifiers \citep{lee2017training}, self-supervised approaches \citep{sehwag2021ssd}, and deep generative models \citep{ran2020bigeminal}.
Nonetheless, such methods can only be performed under the assumption that an additional outlier dataset exists, limiting the discussion to the restricted environment.

Recent approaches detect OOD data using samples generated through the perturbation of input images  \citep{devries2018learning, lee2018simple, ren2019likelihood, choi2019novelty, ran2020detecting}.
For example, \citep{ren2019likelihood} suggest a likelihood ratio test (LRatio) that evaluates the difference between the estimated likelihoods of two generative models, with one capturing the overall data statistic, and the other the background statistic by randomly perturbed input data.
The background statistic is obtained by feeding a contaminated input, which is made by replacing the pixel value with a uniformly sampled value.
\citep{choi2019novelty} construct an OOD detector based on random network distillation (RND) \citep{burda2018exploration} that distinguishes between the representation of in-distribution data and its blurred version.
Specifically, one RND network $g_0$ encodes the input data, and the other RND network $g_1$ encodes blurred input.
The target network $T$ learns to obtain encoded representations.
Then, in the evaluation phase, the further $T(x)$ and $g_0(x)$ are apart, the closer sample $x$ is to OOD data.
Various blurring methods have been applied in \citep{choi2019novelty}: Gaussian blurring (GB), discrete cosine transform blurring (DCT), and singular value decomposition blurring (SVD).
We observe that no robust perturbation method that successfully operates on all datasets exists, as shown in Table \ref{table 1} in Section \ref{section 3.2}, necessitating a new OOD detection method without any additional perturbed data generation.

Concurrently, several works propose OOD detectors that do not assume the nature of the OOD samples \citep{choi2018waic, serra2019input, nalisnick2019detecting, xiao2020likelihood, havtorn2021hierarchical}.
\citep{choi2018waic} measure uncertainty by combining the ensemble method with the Watanabe Akaike information criterion, which is known to be relatively stable on the generative adversarial networks (GANs) \citep{goodfellow2014generative}.
Input complexity (IC) \citep{serra2019input} suggests a lossless compressor to get the background likelihood of a sample.
Precisely, IC measures the difference between the likelihoods induced from the generative model and background likelihood from the lossless compressor.
Although these methods are statistically well-clarified, they obtain substandard results on VAE, which implies that such approaches are not robust to various models \citep{choi2018waic, xiao2020likelihood}.
\citep{nalisnick2019detecting} propose an explicit statistical test for their evaluation metric.
This method uses traditional robust statistics, which require more than one test sample. Thus, such an approach is limited to group anomaly detection.

Likelihood regret (LR) \citep{xiao2020likelihood}, the most relevant approach to our work,
compute the log ratio between the likelihood approximated by the original VAE and its parameter optimized for each input.
Despite the comparatively superior performance, their method is inapplicable to other generative models because of its high reliance on the bottleneck structure of VAEs.
Moreover, optimizing a deep neural network with millions of parameters for every sample is time-consuming.
Thus, a new time-efficient OOD score applicable to a broader range of deep probabilistic generative models needs to be designed.


\section{Background and Problem Formulation} \label{section 3}

\begin{figure}[t]
\begin{subfigure}[]{0.245\textwidth}
  \centering
  \includegraphics[width=1\textwidth]{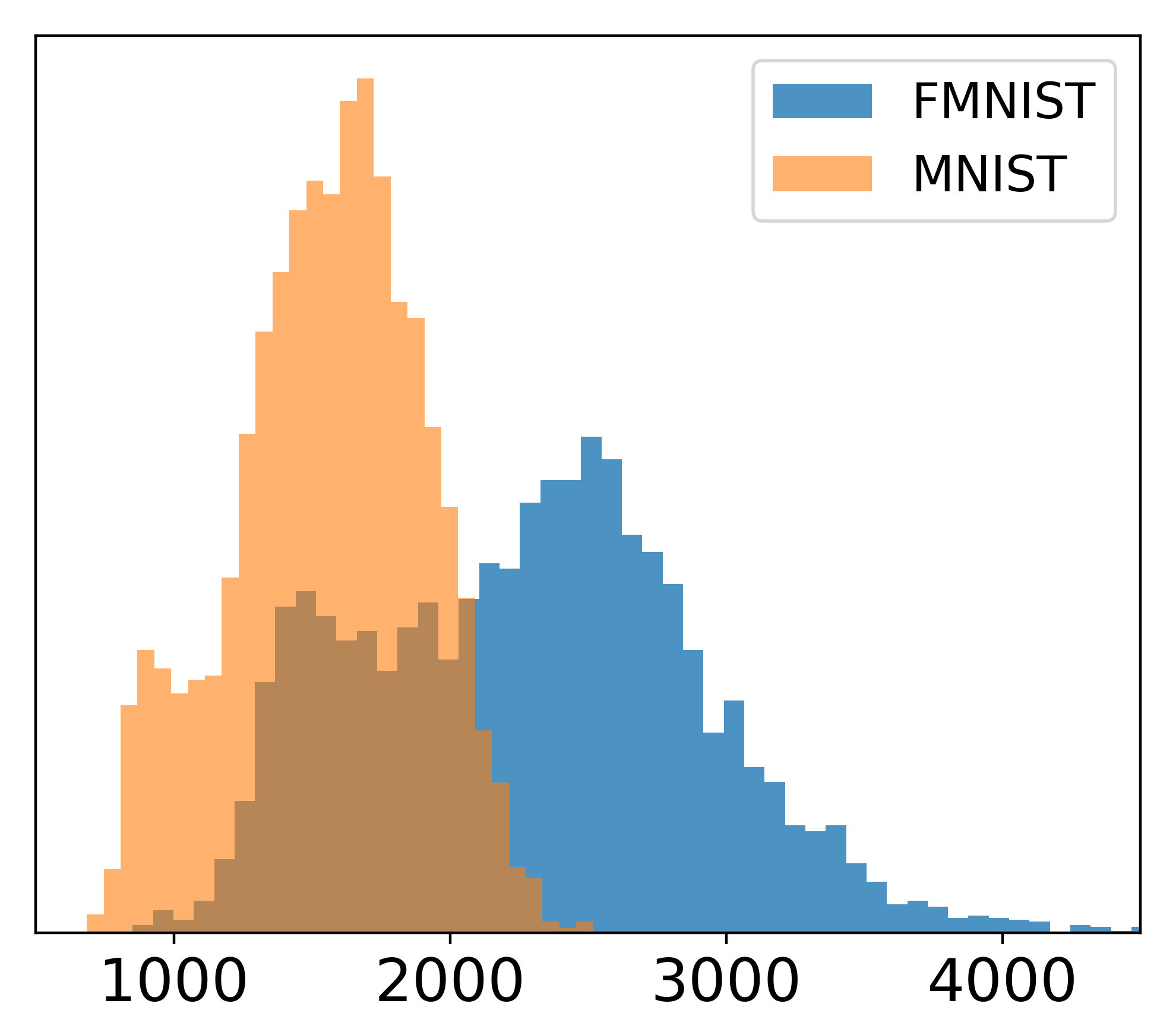}
  \caption{}
\end{subfigure}
\begin{subfigure}[]{0.245\textwidth}
  \centering
  \includegraphics[width=1\textwidth]{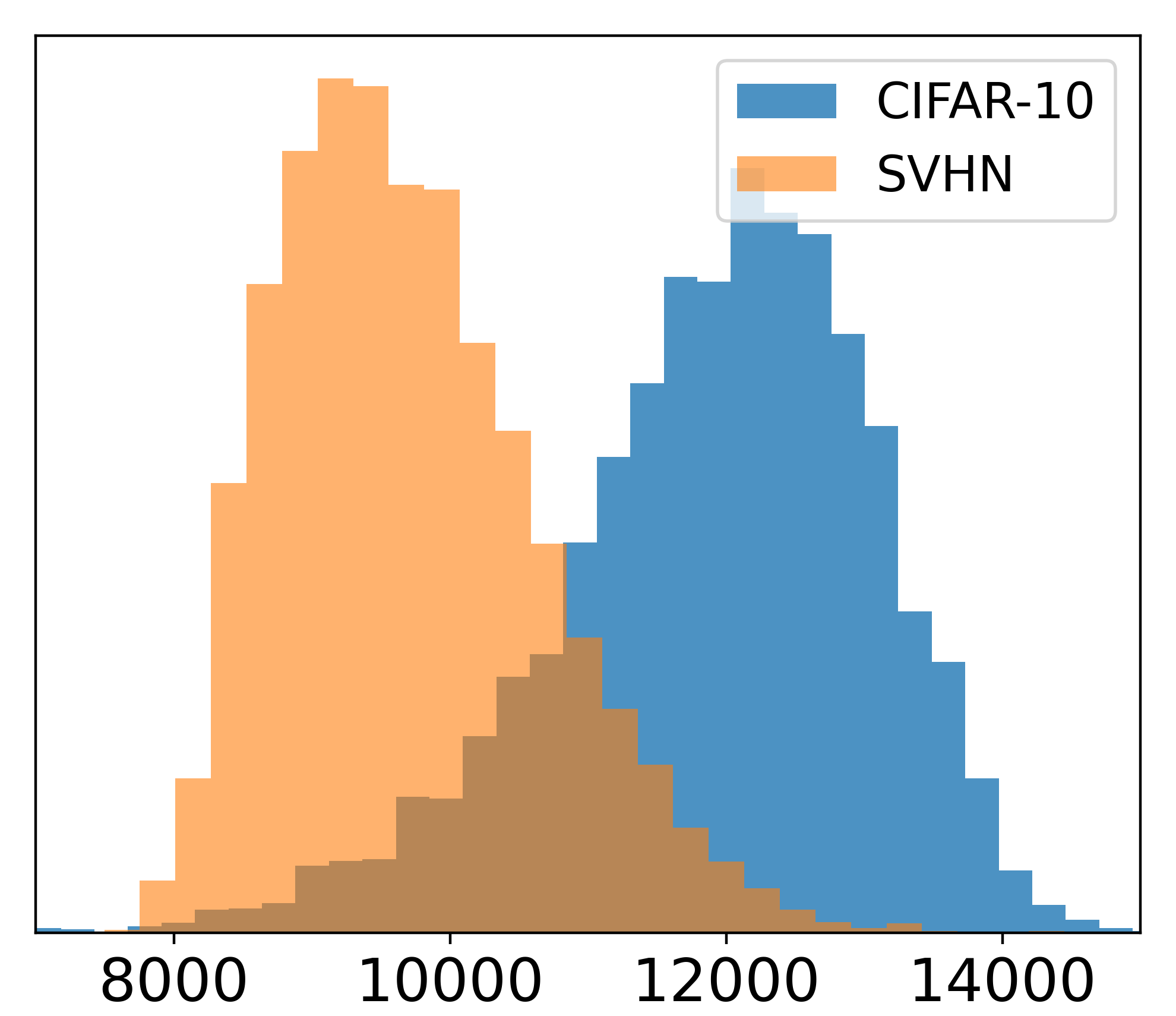}
  \caption{}
\end{subfigure}
\begin{subfigure}[]{0.245\textwidth}
  \centering
  \includegraphics[width=1\textwidth]{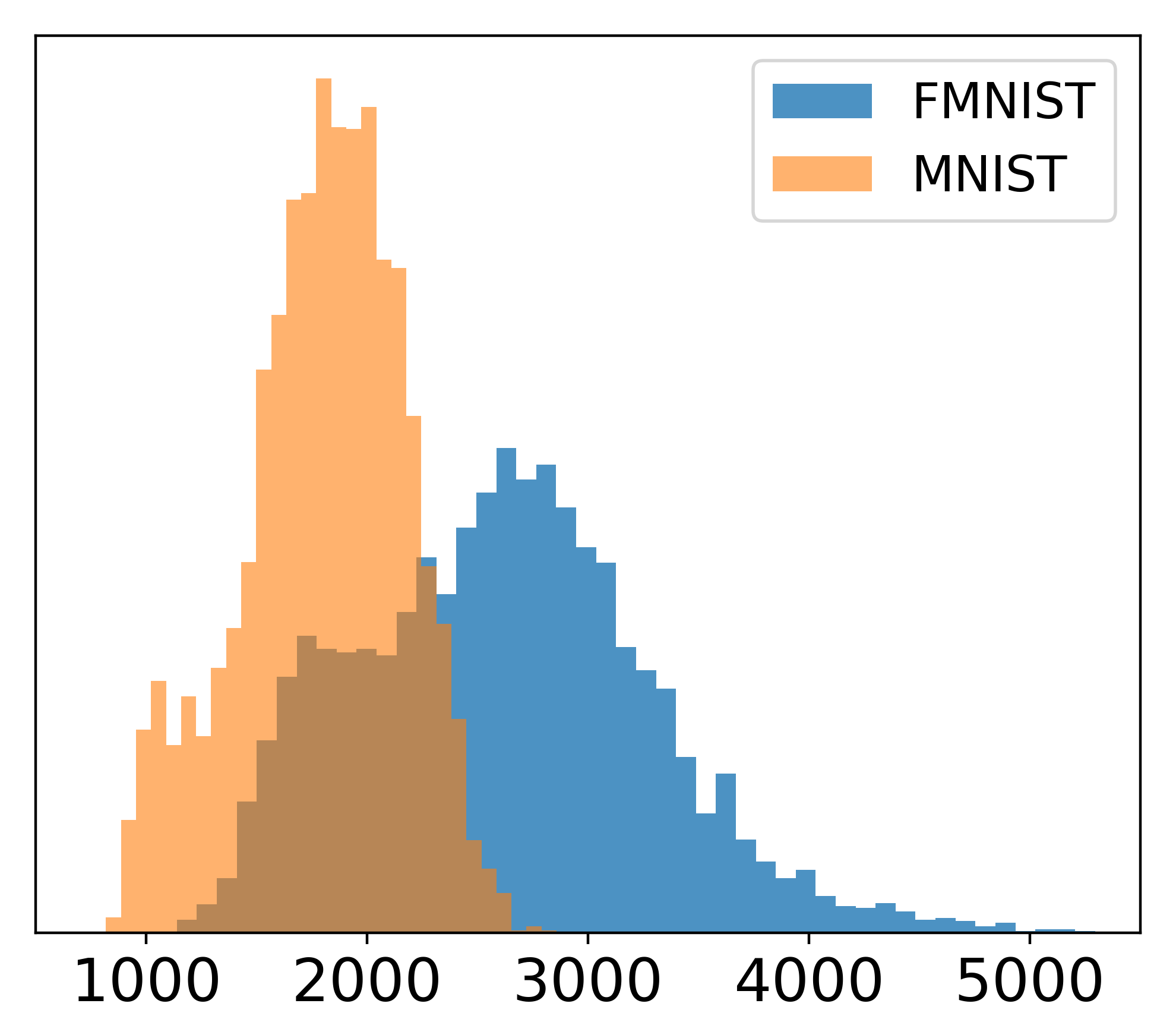}
  \caption{}
\end{subfigure}
\begin{subfigure}[]{0.245\textwidth}
  \centering
  \includegraphics[width=1\textwidth]{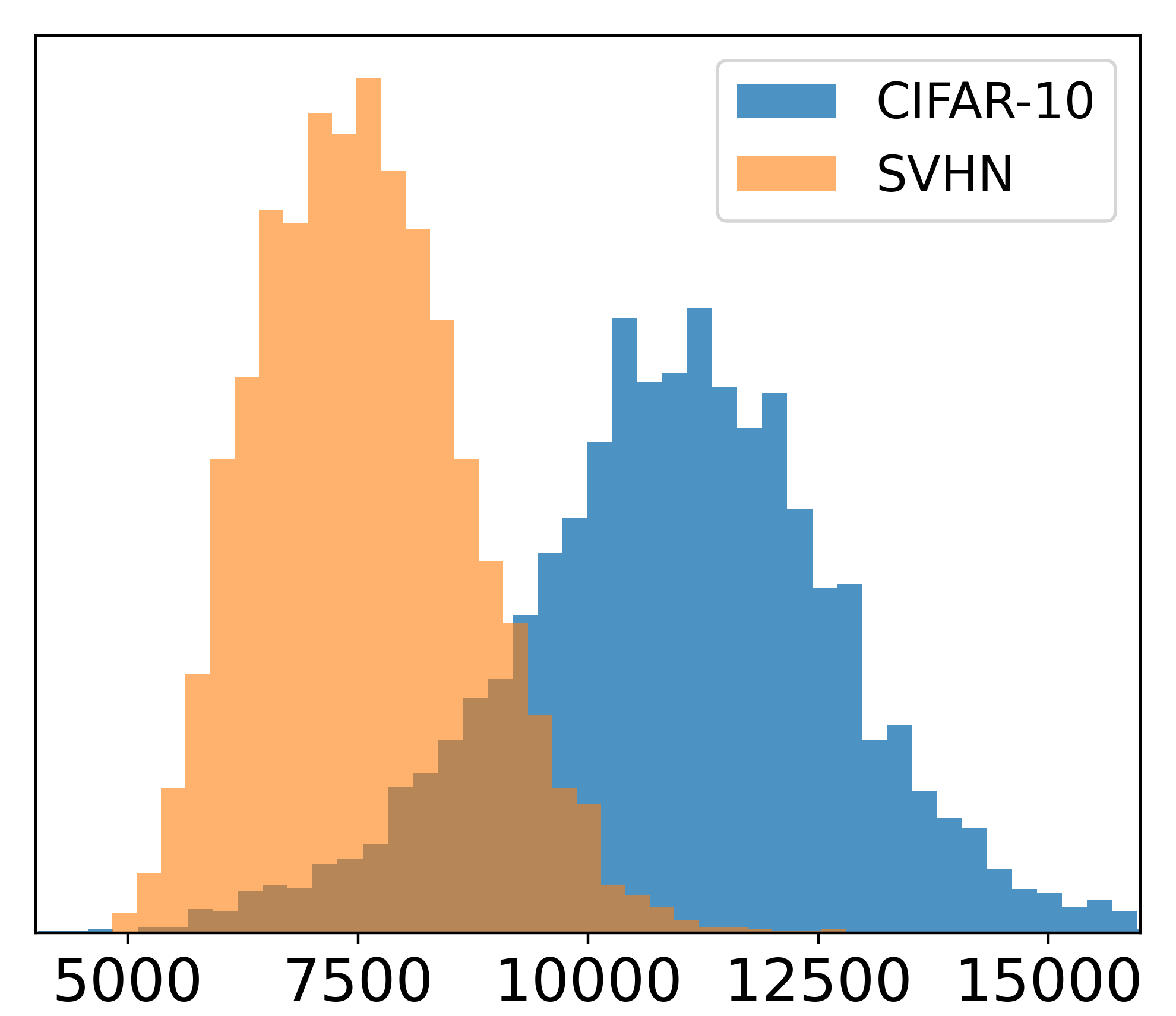}
  \caption{}
\end{subfigure}
\caption{Histograms of the negative log likelihoods (NLLs) obtained from:
(a) Fashion-MNIST vs. MNIST with VAE trained on Fashion-MNIST,
(b) CIFAR-10 vs. SVHN with VAE trained on CIFAR-10,
(c) Fashion-MNIST vs. MNIST with GLOW trained on Fashion-MNIST,
(d) CIFAR-10 vs. SVHN with GLOW trained on CIFAR-10.
Note that NLLs of OOD data are lower than NLLs of in-distribution data.
}
\label{figure 1}
\end{figure}

\subsection{Difficulties in OOD detection with Probabilistic Generative Models}
\label{section 3.1}
Deep probabilistic generative models learn to maximize the likelihood of the input data and estimate the data distribution. 
Thus, they have been previously expected to be accurate detectors that only assign high likelihoods for in-distribution data.
In this light, anomaly detection with probabilistic generative models, especially for VAEs, has been widely applied to handle real-world tasks in various fields \citep{an2015variational, xu2018unsupervised, kipf2016variational, malhotra2016lstm, zhang2019deep}.
However, most deep probabilistic generative models, including auto-regressive models, VAEs, and normalizing flows, assign abnormally high likelihoods to certain OOD samples \citep{nalisnick2018deep}. 
To verify this, we trained our VAE and GLOW on Fashion-MNIST (FMNIST) and CIFAR-10.
The implementation settings are described in Appendix \ref{Appendix B : Our Experimental Settings}.
Figure \ref{figure 1} (a, b) shows the negative log-likelihoods (NLLs) obtained from VAEs trained on (a) Fashion-MNIST and (b) CIFAR-10.
Similarly, Figure \ref{figure 1} (c, d) shows NLLs obtained from GLOWs trained on (c) Fashion-MNIST and (d) CIFAR-10.
Histogram of NLLs on Fashion-MNIST and MNIST data samples are presented in Figure \ref{figure 1} (a, c), and histogram on CIFAR-10 and SVHN are shown in Figure \ref{figure 1} (b, d).
All the histograms show that OOD datasets such as MNIST and SVHN tend to have lower NLL than the in-distribution datasets.
This is a common phenomenon even for other deep probabilistic generative models, including real-valued non-volume preserving network (real-NVP) \citep{dinh2016density, kirichenko2020normalizing} or PixelCNN++ \citep{nalisnick2018deep, nalisnick2018deep}.
Moreover, shown in Appendix \ref{Appendix F : Images}, VAE and GLOW trained on Fashion-MNIST and CIFAR-10 reconstruct MNIST and SVHN samples in high quality.
These abnormal phenomena make OOD detection through likelihood-based generative models notoriously challenging.

\subsection{Why should we detect OOD samples without Outlier Exposure} \label{section 3.2}
The majority of recent works suggest a detector that distinguishes between input data and their perturbation.
These approaches conjecture that such separation may be generalized in arbitrary abnormal samples.
However, the entire image manifold contains uncountably many OOD samples with different intrinsic properties.
Hence, although such an outlier-exposing technique performs better on specific data with similar properties to the generated outliers, we expect it may work poorly on others.

To confirm this, we conduct an experiment that evaluates the proposed methods in LRatio \citep{ren2019likelihood} and RND \citep{choi2019novelty}, the representative outlier-exposing approaches.
We analyze the robustness of those methods with VAE trained on CIFAR-10 in four perturbation methods proposed in \citep{choi2019novelty} and \citep{ren2019likelihood}; Uniformly random noise (Random), Gaussian blurring (GB), discrete cosine transformation (DCT) blurring, and singular value decomposition (SVD) blurring.
These four types of perturbation mainly focus on blurring or noising, which does not change the overall brightness of an image.
Accordingly, we investigate changes in the performance while adjusting the brightness of OOD samples.
Table \ref{table 1} presents the result of the experiment, considering CIFAR-10 as in-distribution data and the CelebA dataset as OOD data.
Each row indicates the area under the receiver operating characteristic curves (AUROCs) for each blurring, and each column denotes AUROCs for different changes in brightness of the CelebA dataset.
Here, 1$\times$ denotes the default brightness of each sample, and we control it from 0.2$\times$ to 1.8$\times$.
The standard deviation of AUROC shows the robustness of methods in the variation of brightness, while the mean of AUROC shows the overall performance.
Hence, Table \ref{table 1} illustrates that LRatio and RND are relatively unstable to our proposed metric ROSE, scored the mean 0.883 and the standard deviation 0.074 of AUROC. 
Thus, we can experimentally confirm that such outlier exposure methods work well only in certain situations and may not in unexpected OODs.
Moreover, we further identify that these approaches require delicate hyper-parameter tuning to obtain an adequate degree of perturbation (See Appendix \ref{Appendix C : Additional Results}).
In conclusion, all of the above observations show why we should not use outlier exposure when detecting OOD samples.

\begin{table}[t]
    \setlength{\tabcolsep}{2.5pt}
    \centering
    {\fontsize{10}{13}\selectfont  
    \aboverulesep=0ex 
    \belowrulesep=0ex 
    \begin{tabular}{llccccccccc|cc}
        \toprule
        \multicolumn{2}{c}{Trained on CIFAR-10} & \multicolumn{9}{c}{Brightness of CelebA} & \multirow{2}{*}{Mean} & \multirow{2}{*}{Std} \\
        \cmidrule(lr){3-11}
        Metrics & Perturbation& 0.2$\times$ & 0.4$\times$ & 0.6$\times$ & 0.8$\times$ & 1$\times$ & 1.2$\times$ & 1.4$\times$ & 1.6$\times$ & \multicolumn{1}{c}{1.8$\times$} & &\\
        \midrule
        \multirow{4}{*}{LRatio}
            & SVD & 0.575 & 0.557 & 0.440 & 0.388 & 0.324 & 0.564 & 0.630 & 0.593 & 0.524 & 0.511& 0.097\\
            & DCT & 0.822 & 0.746 & 0.562 & 0.461 & 0.373 & 0.463 & 0.459 & 0.416 & 0.371 &0.519&0.153\\
            & GB & 0.941 & 0.777 & 0.657 & 0.577 & 0.458 & 0.087 & 0.019 & 0.007 & 0.004 &0.392&0.348\\
            & Random & 0.010 & 0.085 & 0.205 & 0.269 & 0.333 & 0.564 & 0.721 & 0.811 & 0.870 &0.430&0.302\\
        \midrule
        
        \multirow{4}{*}{RND}
            & SVD & 0.183 & 0.075 & 0.120 & 0.358 & 0.566 & 0.733 & 0.825 & 0.871 & 0.893 &0.514&0.316\\
            & DCT & 0.516 & 0.064 & 0.196 & 0.335 & 0.548 & 0.737 & 0.836 & 0.880 & 0.896 &0.556&0.289\\
            & GB & 0.053 & 0.073 & 0.187 & 0.335 & 0.554 & 0.748 & 0.849 & 0.897 & 0.917 &0.513&0.337\\
            & Random & 0.010 & 0.037 & 0.140 & 0.303 & 0.577 & 0.794 & 0.892 & 0.933 & 0.950 &0.515&0.374\\
        \midrule
        ROSE & & 0.954 & 0.890 & 0.818 & 0.774 & 0.778 & 0.871 & 0.928 & 0.960 & 0.976 &\textbf{0.883}&\textbf{0.074}\\
        \bottomrule
    \end{tabular}
    }
    \vspace{2mm}
    \caption{AUROCs of LRatio \citep{ren2019likelihood}, RND \citep{choi2019novelty} and ROSE(ours) on CelebA with various brightness. All models are trained on CIFAR-10. Mean stands for the mean of AUROCs in a row, and Std indicates the standard deviation of AUROCs.}
    \label{table 1}
\end{table}

\begin{figure}[t]
    \centering
    \includegraphics[width=0.7\textwidth]{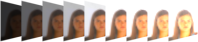}
\caption{Examples of CelebA images with various brightness.}
\label{figure 2}
\end{figure}

\section{Proposed Method} \label{section 4}
We first hypothesize that our trained model parameter is optimal and then approximate the marginal likelihood ratio statistic of each layer by second-order approximation (Section \ref{section 4.1}).
Next, we present a novel approximation method, called EKFAC \citep{george2018fast}, which allows faster computation of the inverse matrix by reducing the large-scale matrix to tensor product (Section \ref{section 4.2}).
Finally, we suggest our metric, ROSE, the aggregation of the layer-wise scores (Section \ref{section 4.3}).

\subsection{Assessing the Impact on the Marginal Likelihood}
\label{section 4.1}
Let $\theta_0^l$ be a set of parameters of $l$-th layer or module that maximizes the marginal likelihood of given in-distribution data $X=\{x_i\}_{i=1}^{N}$ and let $\theta_1^l$ be its counterpart that maximizes marginal likelihood of $X\cup \{\Tilde{x}\}$ where $\Tilde{x}$ is a given test sample.
For simplicity, in this section, we denote $\theta_i^l$ as $\theta_i$.
Our objective is to quantify how much the model parameters $\theta_{0}$ should be improved when the test sample $\Tilde{x}$ is added to the in-distribution data $X$.
Our high-level intuition is that if $\Tilde{x}$ is an OOD sample, a more significant improvement is demanded, which enables OOD detection.
We induce our improvement score starting from the marginal likelihood ratio statistic defined as the following :
\begin{equation} 
\label{likelihood ratio}
    \mathcal{S}_X(\Tilde{x}) = \log{p(\theta_1|X, \Tilde{x})} - \log{p(\theta_0|X, \Tilde{x})}
\end{equation}
The biggest problem with \eqref{likelihood ratio} is that we must optimize $\theta$ for every test sample, which is tremendously time-consuming. Hence, we approximate \eqref{likelihood ratio} with second-order analysis to obtain a tractable form. 
\begin{prop}
\label{prop 1}
    The score $\mathcal{S}_X(\Tilde{x})$ is approximated as
    \begin{equation}
        \mathcal{S}_X(\Tilde{x})
        = \nabla_\theta \log{p(\Tilde{x}|\theta_0)}^T \Delta\theta
        +\frac{1}{2}\Delta\theta^T \big(\nabla_\theta^2 \log{p(\Tilde{x}|\theta_0)}
        - N \mathcal{F}_{X}(\theta_0)\big)\Delta\theta +O(\lVert \Delta\theta\rVert^3),
    \end{equation}
    where $\Delta\theta = \theta_1-\theta_0$, $N=|X|$, and $\mathcal{F}_X(\theta_0)=\mathbb{E}_{x\in X}\big[\nabla_\theta \log{p(x|\theta_0)} \nabla_\theta \log{p(x|\theta_0)}^{T}\big]$.
\end{prop}
\begin{proof}
    Direct application of the Laplace approximation. For proof, see Appendix \ref{Appendix E : proofs}.
\end{proof}

Proposition \ref{prop 1} shows that the optimal direction of $\Delta\theta^*$ using second order approximation is
\begin{equation}
    \Delta\theta^*
    = \big(N \mathcal{F}_X (\theta_0)-\nabla_\theta^2 \log{p(\Tilde{x}|\theta_0)}\big)^{\minus1}
    \nabla_\theta \log{p(\Tilde{x}|\theta_0)}.
\end{equation}
Since $N$ is a subset of training set which is sufficiently large, we approximate $N \mathcal{F}_X (\theta_0)-\nabla^2_\theta \log{p(\Tilde{x}|\theta_0)} \approx N \mathcal{F}_X (\theta_0)$.
Then, we obtain
\begin{equation}
    \mathcal{S}_X(\Tilde{x})
    \propto \frac{1}{N}\nabla_\theta \log{p(\Tilde{x}|\theta_0)}^T \mathcal{F}_X(\theta_0)^{\minus1} \nabla_\theta \log{p(\Tilde{x}|\theta_0)},
\end{equation}
by approximating $\theta_{1} \approx \theta_{0} + \Delta\theta^{*}$.
Therefore, we obtain:
\begin{equation} \label{ROSE}
    \mathcal{S}_X(\Tilde{x}) = s(\Tilde{x};\theta_0)^T \mathcal{F}_X(\theta_0)^{\minus1} s(\Tilde{x};\theta_0),
\end{equation}
up to scaling where $s(\Tilde{x};\theta_0) = \nabla_\theta \log{p(\Tilde{x}|\theta_0)}$.
Our approximation decouples the computation of the in-distribution set $X$ and test sample $\Tilde{x}$. Hence, once we compute the inverse of the Fisher matrix $\mathcal{F}_X(\theta_0)$ from set $X$, we directly obtain $\mathcal{S}_X$ for each test sample by backpropagating only once.
In addition, several statistical approaches \citep{nalisnick2019detecting}, including the likelihood-ratio test, require a group of samples, leading to group anomaly detection; however, our work uses only a single test sample to obtain the detection score.

\subsection{Approximation of  the Fisher Information Matrix} \label{section 4.2}
Despite the presence of \eqref{ROSE}, computing the exact inverse of the Fisher information matrix $\mathcal{F}(\theta) \in \mathbb{R}^{k \times k}$ remains challenging since the number of parameters $k$ in a typical deep generative model exceeds millions.
To address this issue, we employ two branches of studies; Simple diagonal approximation \citep{kirkpatrick2017overcoming} and more sophisticated factored approximation such as Kronecker-factored approximate curvature (KFAC) \citep{grosse2016kronecker} or eigenvalue-corrected Kronecker factorization (EKFAC) \citep{george2018fast}.
Both of the approaches extremely reduce memory consumption and computational cost.

First of all, the diagonal preconditioning is one of the simplest ways to compute an inverse of a large-scale matrix.
Through this method, $\mathcal{F}_X(\theta_0)$ can be directly approximated by:
\begin{equation} 
    \mathcal{F}_X(\theta_0) \approx \mathbb{E}_{X}\big[\mathrm{diag}\,\big(s(x;\theta_0) s(x;\theta_0)^{T}\big)\big],
\end{equation}
where $\mathrm{diag}(\cdot)$ is the diagonal part of a matrix.
Such preconditioning enables us to estimate variances along the initial parameter basis, which may be quite different from the eigenbasis.

For better approximation, we consider EKFAC \citep{george2018fast}, which accelerates computation by reducing matrix into tensor product, to estimate an orthogonal matrix $U$ and a diagonal matrix $\Sigma$ which satisfy the following:
\begin{equation}
    \mathcal{F}_X(\theta_0) = \mathbb{E}_{X}\big[s(x;\theta_0)s(x;\theta_0)^T\big] = U\,\Sigma\,U^T.
\end{equation}
To simplify our discussion, assume that $\theta_0$ is a weight matrix of a fully-connected layer which receives an input $h$ and returns a pre-activation $a=\theta_0^T h$.
By letting $\delta = \partial{l}/\partial{a}$ where $l$ is the loss function, we obtain
\begin{equation}
    \mathbb{E}_{X}\big[s(x;\theta_0)s(x;\theta_0)^T\big] = \mathbb{E}_{X}\big[(h\delta)(h\delta)^T\big] \approx A \otimes B,
\end{equation}
where $A = \mathbb{E}_X\big[hh^T\big]$, $B = \mathbb{E}_X\big[\delta\delta^T\big]$, and $\otimes$ denotes the Kronecker product.
Here, we approximate $U \approx U_A\otimes U_B$ where $U_A$ and $U_B$ are orthogonal matrices obtained from eigen-decomposition of $A$ and $B$, respectively.
Meanwhile, an elementary identity $(A\otimes B) \mathrm{vec}(C) = \mathrm{vec}(B^{T}CA)$ deduce the following:
\footnote{For details, see \citep{george2018fast}.}
\begin{equation}
    \Sigma \approx \mathbb{E}_X\big[\big((U_A\otimes U_B)\, s(x;\theta_0)\big)^2\big] = \mathbb{E}_X\big[\mathrm{vec}\big(U_B^T\, s(x;\theta_0)\, U_A\big)^2\big].
\end{equation}
Now we have $U_A$, $U_B$ and $\Sigma$, the compositional components of $\mathcal{F}_X(\theta_0)$ in hand. Finally, \eqref{ROSE} is calculated by using elementary properties of tensor product.
For multiple layers, we approximate it independently, which results in layer-wise block matrix. For the precise algorithm, see Appendix \ref{Appendix A : Additional details}.

\subsection{Aggregation of Layer-wise Scores} \label{section 4.3}
Let $\mathcal{S}_X^l$ be the score defined on \eqref{ROSE}, with the parameters of the $l$-th layer. For arbitrary given test sample $\Tilde{x}$, let $s^l = \mathcal{S}^l_X(\Tilde{x})$ and $s = (s^1, s^2, ..., s^L)$ where $L$ is the number of layers of the model.
Although the aggregated score can be obtained by simply adding, this should be prohibited, since the scales of each $s^l$ are different.
Then, there remains two choices; Selecting adequate layers for OOD detection or aggregating scores by normalizing each component of $s$.
The former case is discussed in Appendix \ref{Appendix D : Parameter Selection}.
For the latter, we simply normalize $s^l$ by the following means and variances:
\begin{equation}
    \mu^l = \mathbb{E}_{ X}\big[\mathcal{S}_X(x;\theta^l)\big],\quad
    \sigma^l = \sigma_{X}\big[\mathcal{S}_X(x;\theta^l)\big].
\end{equation}
Now, let $\hat{s}^l = (s^l-\mu^l)/\sigma^l$ for all $l \in \{ 1, 2, ..., L\}$ and $\hat{s} = (\hat{s}^1, \hat{s}^2, ..., \hat{s}^L)$.
Note that if every component of $\hat{s}$ is sufficiently small, then it is considered as an in-distribution sample.
Therefore, inspired by the shallow neural network architecture \citep{NIPS2012_c399862d}, we define our score as
\begin{equation} \label{rose_final}
    \mathrm{ROSE}(\Tilde{x}) =\big\lVert \mathrm{ReLU}(\hat{s} + \beta) \big\rVert_p,
\end{equation}
where $p\geq 1$, $\beta$ is an $L$-vector hyper-parameter and $\mathrm{ReLU}$ is a rectified linear units \citep{NIPS2012_c399862d}.
For our experiments, we simply take $\beta=0$ and $p=\infty$.
Again, our detailed algorithm is presented in Appendix \ref{Appendix A : Additional details}.

\section{Experimental Results} \label{section 5}
We conduct experiments on benchmark image datasets to investigate how ROSE performs in the OOD detection task.
For the experiments, we train two types of probabilistic generative models, standard VAE and GLOW, on in-distribution datasets (Fashion-MNIST and CIFAR-10).
Most of our work centers on VAE since no other fast and stable metric has been discovered yet, to the best of our knowledge.
We additionally show that our score discriminates OOD samples in high quality in GLOW.
We emphasize that our score does not use any prior knowledge of OOD samples and is sufficiently fast.
Our implementation on VAE and GLOW follows the existing literature (\citep{xiao2020likelihood}, \citep{nalisnick2018deep}, \citep{kingma2018glow}).
For details of our experimental settings, see Appendix \ref{Appendix B : Our Experimental Settings}.

Figure \ref{figure 3} shows the receiver operating characteristic (ROC) curves of likelihood and ROSE on (a) Fashion-MNIST vs. MNIST and (b) CIFAR-10 vs. SVHN.
For fair evaluation, the same VAE models trained on (a) Fashion-MNIST and (b) CIFAR-10, respectively, were used.
As a result, the AUROC of likelihood-based detection records below 0.2 on both situations (a) Fashion-MNIST vs. MNIST and (b) CIFAR-10 vs. SVHN.
We emphasize that this abnormal phenomenon is not limited to (a) and (b) but shows a problematic performance for most datasets.
Except for the noise, which is perfectly distinguishable with in-distribution samples, the AUROC on the other datasets also achieve less than 0.6 when CIFAR-10 is in-distribution, as shown in Table \ref{table 2}.
In contrast, our score ROSE achieves 0.995 and 0.907 AUROC for (a) Fashion-MNIST vs. MNIST and (b) CIFAR-10 vs. SVHN, respectively.

\begin{figure}[t]
\begin{subfigure}[]{0.49\textwidth}
  \centering
  \includegraphics[width=1\textwidth]{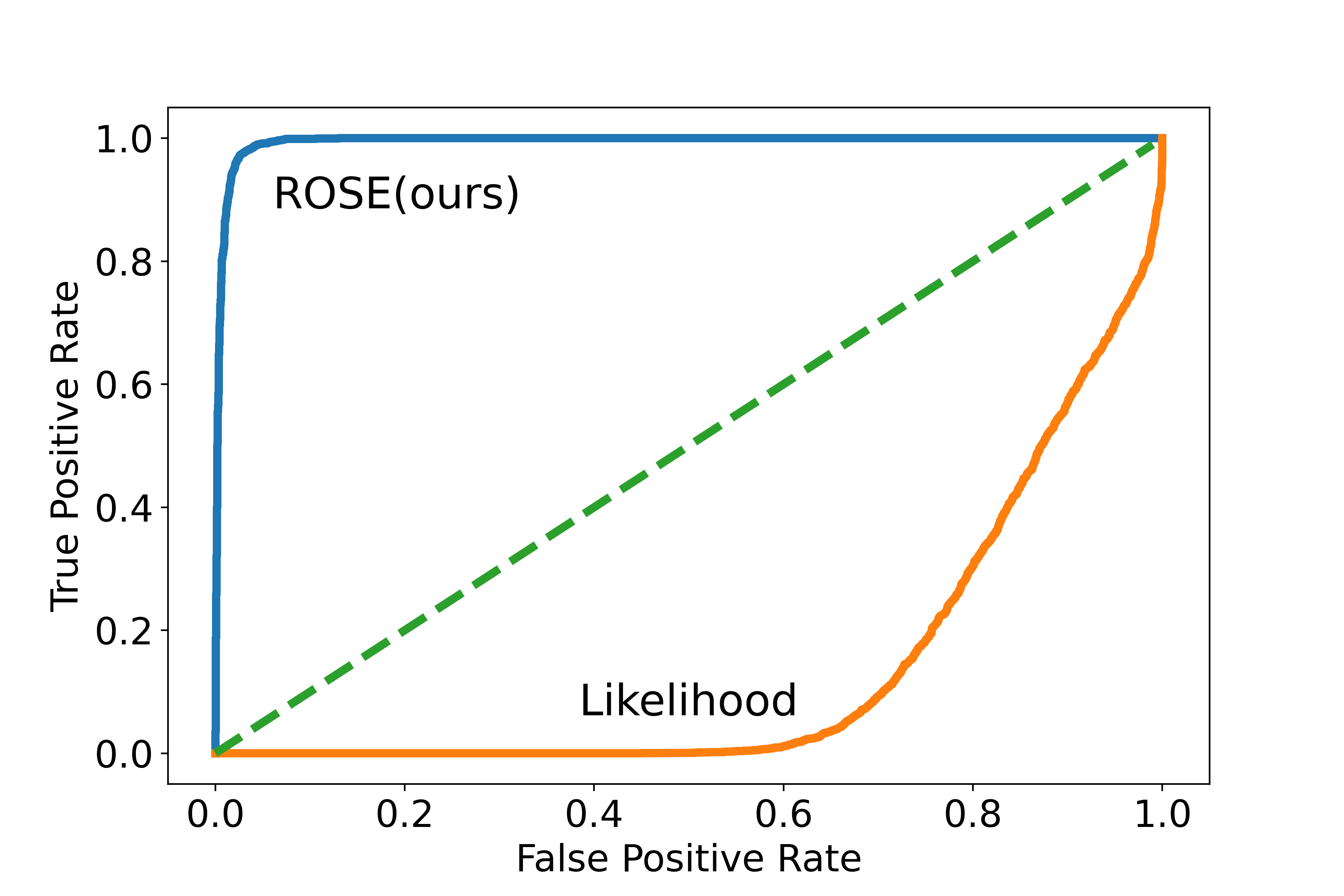}
  \caption{}
\end{subfigure}
\begin{subfigure}[]{0.49\textwidth}
  \centering
  \includegraphics[width=1\textwidth]{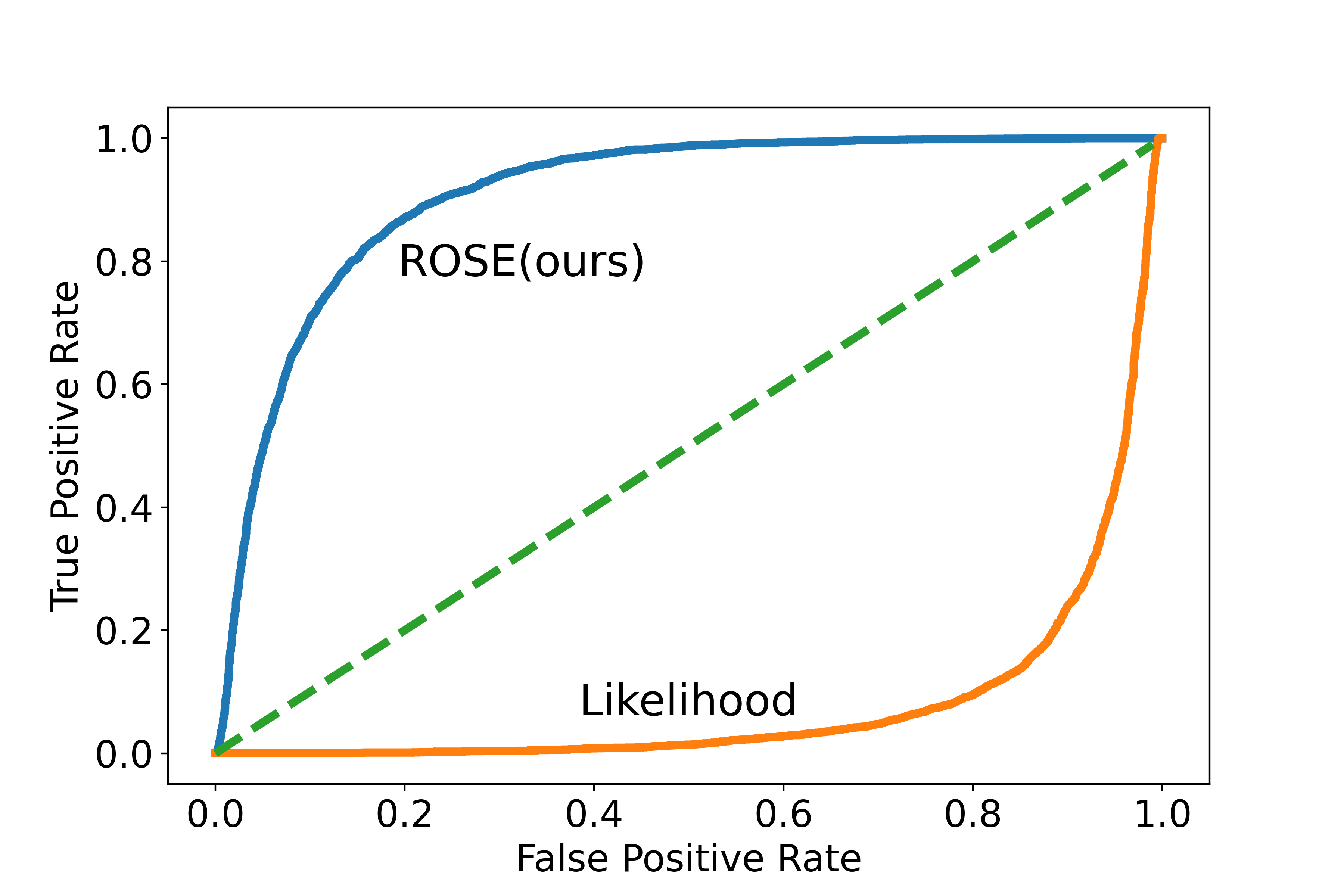}
  \caption{}
\end{subfigure}
\caption{Comparing the ROC curves of using ROSE and the NLL for OOD detection. (a) On Fashion-MNIST vs. MNIST experiment, ROSE improves the AUROC from 0.161 to 0.995. (b) On CIFAR-10 vs. SVHN experiment, ROSE improves the AUROC from 0.081 to 0.907.}
\label{figure 3}
\end{figure}

\subsection{Comparison with Other OOD Scores} \label{section 5.1}

In this section, we compare our method with other benchmark methods.
Our primary evaluation metric is the AUROC, which is a standard measurement for the OOD detection task \citep{hendrycks2018deep}.
We also provide the area under the precision-recall curve (AUPRC) and the false positive rate at 80\% true positive rate (FPR80) as comparison metrics (See Appendix \ref{Appendix C : Additional Results}).
Furthermore, to measure the overall performance of each detection method, we conduct a newly aggregated dataset gathered from the benchmark datasets.
The `overall' dataset is collected from MNIST \citep{mnist}, KMNIST \citep{kmnist}, NotMNIST, Omniglot \citep{omniglot}, SVHN \citep{svhn}, CelebA \citep{celeba}, LSUN \citep{lsun}, noise, and constant by one thousand each.
\footnote{All samples are selected uniformly at random.}
Additionally, for Fashion-MNIST experiments, we add one thousand samples of CIFAR-10 to the overall dataset, and for CIFAR-10 experiments, we add one thousand samples of Fashion-MNIST.
For a detailed description of the datasets we used, see Appendix \ref{Appendix F : Images}.

For the comparison on VAE, we use Likelihood Regret (LR) \citep{xiao2020likelihood}, Likelihood Ratio (LRatio) \citep{ren2019likelihood}, two variants of Input Complexity (IC) \citep{serra2019input}, and NLL.
We have carefully experimented with the other benchmark methods, following the existing literature of \citep{xiao2020likelihood, kingma2018glow}.
Throughout all the experiments on VAE, we adopt importance weighted sampling to approximate the exact likelihood, which is intractable in VAE \citep{burda2015importance}.

The left part of Table \ref{table 2} shows that our proposed score ROSE on VAE successfully corrects the failure of likelihood and achieves state-of-the-art (SOTA) performance across various image domains.
Indeed, for the VAE trained on Fashion-MNIST, all AUROC values are higher than 0.99, indicating the best results except for the noise data. 
For the VAE trained on CIFAR-10, ROSE shows SOTA performance for more than half of the datasets, especially achieving the highest AUROC 0.918 for the overall dataset.
However, despite its overall superior performance, a poor result in CIFAR-10 vs. LSUN shows some limitations of our method.
Moreover, our score performs better on Fashion-MNIST than CIFAR-10.
We guess that the reason for this phenomenon is that the parameters of the VAE trained on Fashion-MNIST are closer to the maximum likelihood estimation, due to the simplicity of training data.
It is also noteworthy that ROSE is stable across all datasets.
For instance, in the CIFAR-10 task, LRatio fails on SVHN/CelebA/constant, and IC(png) fails on noise, while our ROSE shows stable performance above AUROC 0.5 on all datasets.
It is in line with the experimental results of varying brightness, as discussed in Section \ref{section 3.2}.

For the experiment on GLOW, we exclude LR for the comparison method since it can only be utilized in VAE.
We also exclude the result of NLL from Table \ref{table 2} due to its relatively poor performance.
The right part of Table \ref{table 2} shows that ROSE on GLOW achieves the highest overall performance.
Indeed, for the overall dataset, the AUROC is 0.885 and 0.932 on Fashion-MNIST and CIFAR-10 experiments, respectively. 
In other methods, there are more than one datasets that deteriorate the performance of OOD detection.
For instance, in the Fashion-MNIST task, LRatio on SVHN/constant and IC on KMNIST achieve the AUROC less than 0.5.
On the contrary, ROSE has no failure case in which the AUROC is less than 0.5, indicating its stability.

In summary, we demonstrated the stability of ROSE across all datasets we have seen.
Furthermore, our experiments show that ROSE achieves state-of-the-art overall performance both on VAE and GLOW.
It is the only OOD detection score that can be successfully applied to various deep probabilistic generative models, to the best of our knowledge.

\begin{table}[t]
  \centering
  {\fontsize{8}{10}\selectfont  
  \begin{tabular}{cccccccccccc}
    \toprule 
    \multicolumn{1}{c}{} & \multicolumn{5}{c}{VAE} & \multicolumn{1}{c}{} & \multicolumn{5}{c}{GLOW}\\
    \cmidrule{2-7}
    \cmidrule{9-12}
    FMNIST & ROSE & LR$_{\text{E}}$ & NLL & LRatio & IC(png) & IC(jp2) & & ROSE & LRatio & IC(png) & IC(jp2)\\
    \midrule
    MNIST    & \textbf{0.995} & 0.964 & 0.161 & 0.844 & 0.934 & 0.449 &
             & 0.562 & 0.850 & \textbf{0.922} & 0.616 \\
    KMNIST   & \textbf{0.997} & 0.991 & 0.626 & 0.968 & 0.636 & 0.389 &
             & 0.518 & \textbf{0.930} & 0.496 & 0.281\\
    NotMNIST & \textbf{1.000} & \textbf{1.000} & 0.980 & 0.950 & 0.893 & 0.958 &
    & \textbf{0.896} & 0.500 & 0.542 & 0.521 \\
    Omniglot & \textbf{1.000} & \textbf{1.000} & \textbf{1.000} & 0.869 & \textbf{1.000} & \textbf{1.000} &
    & \textbf{1.000} & 0.869 & \textbf{1.000} & \textbf{1.000} \\
    CIFAR-10 & 0.999 & 0.995 & \textbf{1.000} & 0.916 & 0.899 & 0.999 &
             & 0.976 & 0.974 & 0.772 & \textbf{0.998} \\
    SVHN     & \textbf{1.000} & \textbf{1.000} & 0.999 & 0.624 & 0.999 & \textbf{1.000} &
             & 0.951 & 0.147 & 0.718 & \textbf{0.999} \\
    Noise    & 0.996 & 0.997 & \textbf{1.000} & \textbf{1.000} & 0.443 & \textbf{1.000} &
             & \textbf{1.000} & \textbf{1.000} & \textbf{1.000} & \textbf{1.000} \\
    Constant & \textbf{1.000} & 0.998 & 0.953 & 0.647 & \textbf{1.000} & \textbf{1.000} &
             & 0.996 & 0.005 & \textbf{1.000} & \textbf{1.000} \\
    Overall  & \textbf{0.999} & 0.993 & 0.869 & 0.873 & 0.874 & 0.879 &
             & \textbf{0.885} & 0.704 & 0.811 & 0.846\\ 
    \bottomrule
    \toprule
    CIFAR-10 & ROSE & LR$_{\text{E}}$ & NLL & LRatio& IC(png) & IC(jp2) & & ROSE & LRatio & IC(png) & IC(jp2) \\
    \midrule
    SVHN     & 0.907 & \textbf{0.953} & 0.081 & 0.050 & 0.910 & 0.942 &
             & 0.940 & 0.479 & 0.882 & \textbf{0.969} \\
    CelebA   & \textbf{0.778} & 0.682 & 0.428 & 0.333 & 0.588 & 0.303 &
             & 0.694 & \textbf{0.732} & 0.697 & 0.381 \\
    LSUN     & 0.526 & 0.513 & 0.595 & 0.593 & \textbf{0.660} & 0.322 &
             & 0.730 & \textbf{0.832} & 0.759 & 0.388 \\
    MNIST    & \textbf{0.999} & 0.996 & 0.000 & 0.825 & 0.985 & 0.980 &
             & 0.998 & 0.002 & 0.999 &\textbf{1.000} \\
    FMNIST   & \textbf{0.996} & 0.990 & 0.038 & 0.783 & 0.992 & 0.989 &
             & 0.988 & 0.004 & 0.997 & \textbf{1.000} \\
    NotMNIST & \textbf{0.998} & 0.996 & 0.030 & 0.506 & 0.989 & 0.989 &
             & 0.992 & 0.003 & 0.995 & \textbf{0.998} \\
    Noise    & \textbf{1.000} & 0.997 & \textbf{1.000} & \textbf{1.000} & 0.172 & 0.115 &
             & \textbf{1.000} & \textbf{1.000} & \textbf{1.000} & \textbf{1.000} \\
    Constant & 0.973 & 0.995 & 0.277 & 0.223 & \textbf{1.000} & \textbf{1.000} &
             & \textbf{1.000} & 0.001 & \textbf{1.000} & \textbf{1.000} \\
    Overall  & \textbf{0.918} & 0.913 & 0.246 & 0.610 & 0.826 & 0.763 &
             & \textbf{0.932} & 0.319 & \textbf{0.932} & 0.870 \\ 
    \bottomrule
  \end{tabular}
  }
  \vspace{2mm} 
  \caption{AUROC of robust out-of-distribution score (ROSE) and other OOD detection metrics on the benchmark image datasets. 
  The top part of the table uses a VAE (\textbf{left}) and GLOW (\textbf{right}) model trained on Fashion-MNIST and the bottom part of the table uses a VAE (\textbf{left}) and GLOW (\textbf{right}) model trained on CIFAR-10. AUROC values in each row indicate the ability to distinguish an OOD sample (ex: SVHN or CelebA). For all of our experiments, we employ test samples for in-distribution data (ex: CIFAR-10 test). The results on `overall' OOD dataset in the last row is a mixed dataset that consists of images, which are chosen uniformly at random from various benchmark OOD datasets.}
  \label{table 2}
\end{table}

\subsection{Better Understanding ROSE} \label{section 5.2}
In this section, we investigate the efficiency and robustness of ROSE.
We measure the inference time of ROSE and compare it with other benchmark methods.
We also verify that ROSE is robust to various changes in the capacity of VAE and in the number of given in-distribution samples.
These observations imply that our proposed ROSE highly exceeds all other methods in three aspects; performance, robustness, and time-efficiency. 
Further discussion of ROSE, including another approximation of the Fisher information matrix using the diagonal preconditioner and robustness in changing brightness, are presented in Appendix \ref{Appendix C : Additional Results}.

\textbf{Runtime} 
Figure \ref{figure 4} shows the inference times\footnote{Times are for a single NVIDIA Geforce RTX 3090 GPU and Intel Core i9-10900K CPU.} on VAE of benchmark OOD detection methods, consisting of Input Complexity \citep{serra2019input}, Likelihood Ratio \citep{ren2019likelihood}, Likelihood Regret \citep{xiao2020likelihood}, and our proposed ROSE.
The horizontal axis represents the number of images processed per second, and the vertical axis represents AUROC for the overall OOD dataset.
In the case of Likelihood Regret, we can observe a slow processing time of about 2.5 images per second due to the additional optimization per image, despite its outstanding performance.
On the contrary, ROSE handles about 110 images per second, significantly higher than the baseline of real-time detection, and records the best AUROC, 0.918, among the five benchmark methods.

\textbf{Robustness to the sampling number} 
One concern for ROSE is that the performance may be unstable with respect to the number of samples used to obtain the Fisher matrix.
Thus, to verify the robustness of ROSE, we experiment with computing ROSE ten times for each sampling number, varying it from ten to 10K. 
Figure \ref{figure 5} shows the box plot of AUROC on the overall dataset that we discussed at the beginning of Section \ref{section 5.1}. 
The observations of this experiment confirm that our method is robust to the sampling number.
Even when the sampling number is extremely small, it is astonishing that ROSE still attains good performance.
Furthermore, as the number of sampling increases, the AUROC also increases gradually and converges to one highest result. 

\textbf{Robustness to the capacity of VAE}
We further conduct experiments to observe whether ROSE is robust in the model capacity.
The capacity is defined by width, which indicates the number of channels of VAE.
The result of this experiment is in Table \ref{table 3}; here, 1$\times$ denotes the standard size of VAE that we use throughout this paper.
Table \ref{table 3} shows the variation of AUROC on various datasets as the model capacity changes. The results show that our score is highly stable on such variations.

 \begin{table}
  \centering
  {\fontsize{10}{10}\selectfont  
    \begin{tabular}{ccccccccccc}
    \toprule
    \multicolumn{2}{c}{
    \multirow{2}{*}{Capacity}
    } & \multicolumn{4}{c}{FMNIST} & & \multicolumn{4}{c}{CIFAR-10}\\
    \cmidrule{3-6}
    \cmidrule{8-11}
    \multicolumn{2}{c}{}
    & 0.5$\times$ & 1$\times$ & 2$\times$ & 4$\times$ &
    & 0.5$\times$ & 1$\times$ & 2$\times$ & 4$\times$ \\
    \midrule
    \multicolumn{2}{c}{MNIST}
    & 0.997 & 0.995 & 0.991 & 0.988 & & 1.000 & 0.999 & 0.998 & 0.999\\
    \multicolumn{2}{c}{KMNIST}
    & 0.998 & 0.997 & 0.994 & 0.993 & & 0.999 & 0.999 & 0.998 & 0.999\\
    \multicolumn{2}{c}{NotMNIST}
    & 1.000 & 1.000 & 1.000 & 1.000 & & 0.998 & 0.998 & 0.995 & 0.994\\
    \multicolumn{2}{c}{SVHN}
    & 1.000 & 1.000 & 1.000 & 1.000 & & 0.809 & 0.907 & 0.851 & 0.822\\
    \multicolumn{2}{c}{CelebA}
    & 0.999 & 0.999 & 1.000 & 0.998 & & 0.716 & 0.778 & 0.792 & 0.786\\
    \multicolumn{2}{c}{LSUN}
    & 0.999 & 0.999 & 1.000 & 0.999 & & 0.536 & 0.526 & 0.516 & 0.529\\
    \multicolumn{2}{c}{Noise}
    & 0.999 & 0.996 & 0.993 & 0.992 & & 1.000 & 1.000 & 0.993 & 0.940\\
    \multicolumn{2}{c}{Constant}
    & 1.000 & 1.000 & 1.000 & 1.000 & & 0.990 & 0.973 & 0.923 & 0.919\\
    \bottomrule
    \end{tabular}
  }
  \vspace{2mm} 
  \caption{Robustness of ROSE with respect to the capacity of VAE model.
  Every VAE model is trained on Fashion-MNIST and CIFAR-10, respectively.}
  \label{table 3}
\end{table}

\begin{figure}
\centering
\begin{minipage}{.48\textwidth}
  \centering
  \includegraphics[width=.98\textwidth]{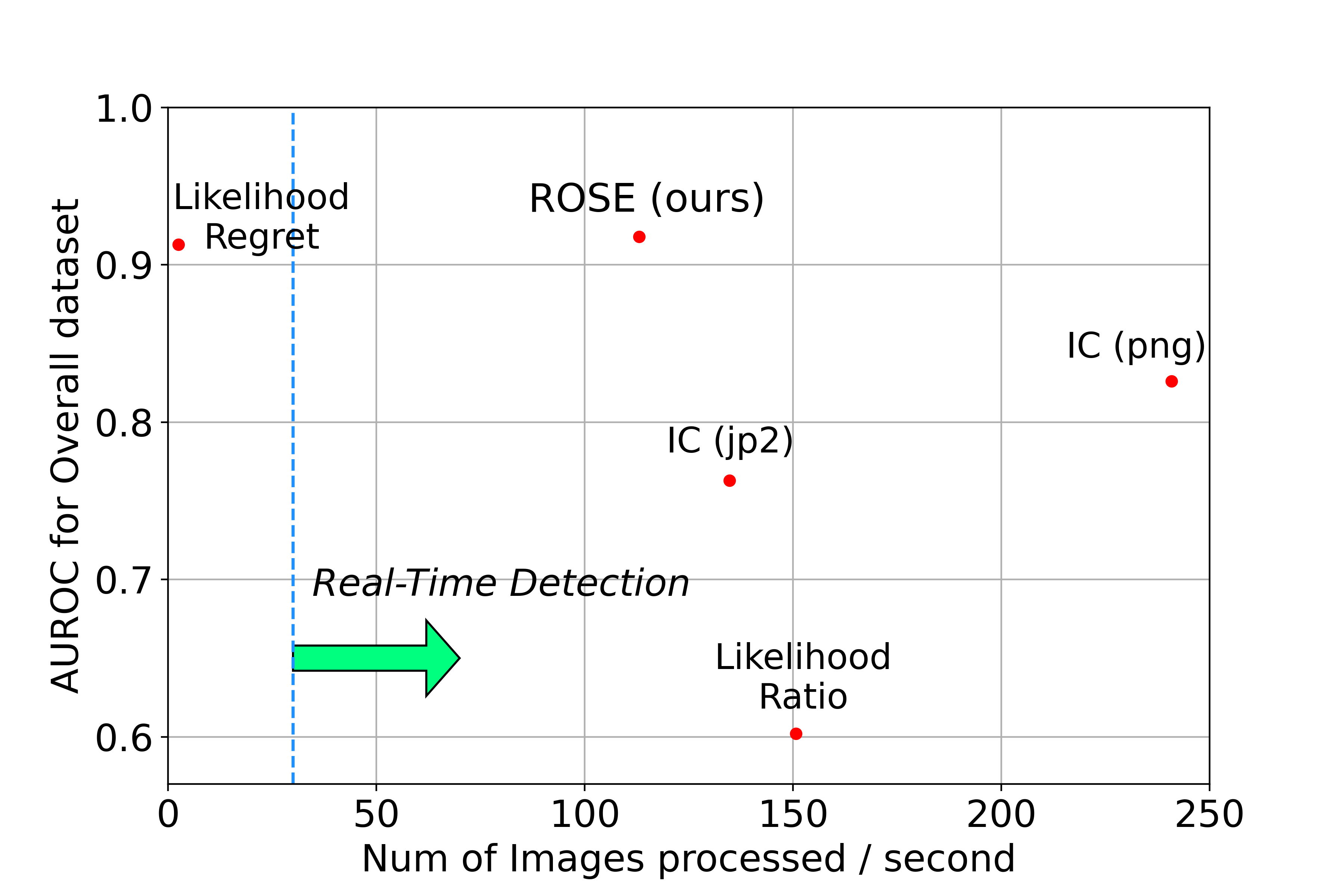}
  \captionof{figure}{Runtime of ROSE and other benchmark OOD detection methods on the same VAE model trained on CIFAR-10.}
  \label{figure 4}
\end{minipage}
\hspace{2mm}
\begin{minipage}{.48\textwidth}
  \centering
  \includegraphics[width=.98\textwidth]{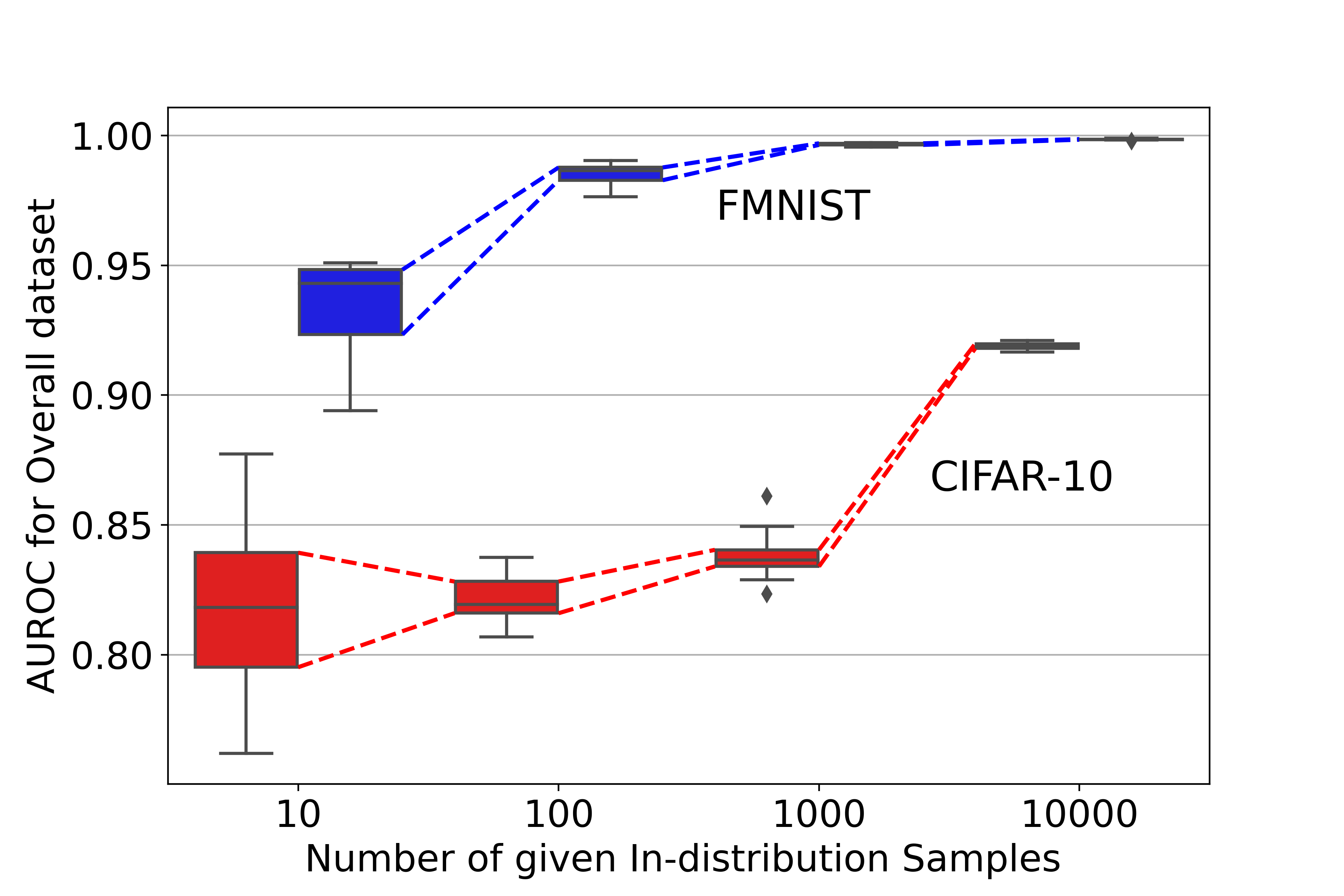}
  \captionof{figure}{AUROCs of ROSE with respect to the number of given in-distribution samples, which are used to gain the Fisher matrix.}
  \label{figure 5}
\end{minipage}
\end{figure}

\section{Conclusion}\label{section 6}
In this paper, we investigate the previous approaches, which rely on generated outliers, and show that such methods are vulnerable.
Accordingly, we propose a new score for OOD detection based on the marginal likelihood ratio statistic, which does not employ outlier exposure.
Second-order approximation of the statistic reformulates our score into the tractable form.
We further apply the EKFAC algorithm to accelerate our method regarding time-efficiency and memory consumption.
As a result, our score shows the best overall performance compared to other benchmark OOD detection methods.
Moreover, the robustness of our metric is demonstrated in a variety of perspectives.
We also propose a custom OOD dataset as a new criterion for a fair evaluation.
We hope that our research will be further developed and applied to a real-world situation.
To the best of our knowledge, our research has no negative societal impact.

\newpage
\bibliographystyle{unsrtnat}
\bibliography{references}

\newpage
\appendix
\section{Additional details about our Proposed Method} \label{Appendix A : Additional details}
In this section, we introduce a pseudo-code computing the inverse of the Fisher information matrix by using the EKFAC algorithm and our final score ROSE. 
By applying 
\begin{equation} \label{tensor}
    (A\otimes B)\mathrm{vec}(C) = \mathrm{vec}(B^{T} C A),    
\end{equation}
it facilitates low memory consumption.
After letting $p=\infty$ in \eqref{rose_final}, computation of ROSE follows Algorithm \ref{algorithm 2}. Here, 
$$\mathcal{F}_{X}^{\minus1} = (U_A\otimes U_B)^T \mathrm{diag}\big({\Sigma^{(l)}}\big)^{\minus1} (U_A\otimes U_B),$$
can be computed by applying \eqref{tensor} to $U_{A}$, $U_{B}$, and $\Sigma$.
For the exact computational algorithm, see \citep{george2018fast}.

\begin{algorithm}[H]
\caption{Preprocessing by EKFAC}
\begin{algorithmic}[1] 
\Require Training samples $X = \{x_{1}, \ldots, x_{N}\}$
\Procedure{Calculate Fisher Matrix}{$X$}
    \For{each layer $l$}
        \State $U_A^{(l)}, S_A^{(l)}  \gets$ eigen-decomposition of $\mathbb{E}_X\big[ h^{(l)} h^{(l)T}\big]$
        \State $U_B^{(l)}, S_B^{(l)}  \gets$ eigen-decomposition of $\mathbb{E}_X\big[ \delta^{(l)} \delta^{(l)T}\big]$
        \State $\Sigma^{(l)}  \gets \mathbb{E}_X\big[\mathrm{vec}\big(U_B^{(l)T}\, s\big(x;\theta_0^{(l)}\big)\, U_A^{(l)}\big)^2\big]$ 
    \EndFor
    \State \textbf{return} $U_A, U_B, \Sigma$
\EndProcedure
\Procedure{Calculate mean and variance}{$X, U_A, U_B, \Sigma$}
    \For{each layer $l$}
        \State $\mu_X^{(l)}  \gets$ $\mathbb{E}_X\big[s\big(x;\theta^{(l)}_0\big)^T (U_A\otimes U_B)^T \mathrm{diag}\big({\Sigma^{(l)}}\big)^{\minus1} (U_A\otimes U_B) s\big(x;\theta^{(l)}_0\big) \big]$
        \State $\sigma_X^{(l)}  \gets$ $\sigma_X\big[s\big(x;\theta^{(l)}_0\big)^T (U_A\otimes U_B)^T \mathrm{diag}\big({\Sigma^{(l)}}\big)^{\minus1} (U_A\otimes U_B) s\big(x;\theta^{(l)}_0\big) \big]$
    \EndFor
    \State \textbf{return} $\mu_X, \sigma_X$
\EndProcedure
\end{algorithmic}
\label{algorithm 1}
\end{algorithm}
\begin{algorithm}[H]
\caption{Computing ROSE}
\begin{algorithmic}[1] 
\Require Test sample $\Tilde{x}$, Inverse of Fisher matrix $\mathcal{F}_{X}^{\minus1}$, layer-wise mean $\mu_{X}^{(l)}$, layer-wise variance $\sigma_{X}^{(l)}$
    \For{each layer $l$}
        \State $\hat{s}^l \gets s\big(\Tilde{x};\theta_0^{(l)}\big)\,\mathcal{F}_{X}^{\minus1}\,s\big(\Tilde{x};\theta_0^{(l)}\big)$
        \State $\hat{s}^l \gets \big(\hat{s}^l-\mu_X^{(l)}\big)\,/\,\sigma_X^{(l)}$
    \EndFor
    \State ROSE$\,\, \gets \big\lVert\mathrm{ReLU}(\hat{s}) \big\rVert_{\infty}$
    \State \textbf{return} ROSE
\end{algorithmic}
\label{algorithm 2}
\end{algorithm}

\newpage
\section{Experimental settings} \label{Appendix B : Our Experimental Settings}

For all experiments on VAE, we used 5000 random test samples of each dataset, and for all experiments on GLOW, we used 1000 random test samples of each dataset.
All input images are resized to 32 $\times$ 32 except LSUN.
For LSUN, due to their irregular rectangular shapes, we resized the shorter side of height and width to 32 while preserving the ratio, then center cropped to get the size of 32 $\times$ 32.

\subsection{VAE} \label{Appendix B.1 : VAE}
Following the existing literature \citep{nalisnick2018deep} and \citep{xiao2020likelihood}, we evaluate two VAEs trained on Fashion-MNIST and CIFAR-10, respectively.
The VAEs are trained for 100 epochs with batch size 64, learning rate 1e-3 using Adam optimizer \citep{kingma2014adam}, decaying it to half every 30 epochs.
The encoder part of VAEs consists of four convolutional layers of kernel size 4, strides 2, without biases.
Also, the channel size starts from $n$ and doubles at each layer, where $n$ is 32 for Fashion-MNIST and $n$ is 64 for CIFAR-10.
The last two layers of the encoder encode a mean and variance, respectively, and their channel sizes are the same as the latent dimension, 100 for Fashion-MNIST and 200 for CIFAR-10.
The decoder part of VAEs has a symmetric structure to the encoder and reconstructs the image pixel-wise.
We apply the EKFAC algorithm and the diagonal preconditioner on all the convolutional layers of the encoder.
The comparison of the EKFAC algorithm and the diagonal preconditioner (DIAG) is presented in Table \ref{table 6}.
For better estimation of likelihood, in LR(E) \citep{xiao2020likelihood}, IC \citep{serra2019input}, and LRatio \citep{ren2019likelihood}, we employ the importance weighted autoencoder (IWAE) likelihood estimation \citep{burda2015importance} with sampling number 20.
The codes for the EKFAC algorithm are obtained from \url{https://github.com/alecwangcq/KFAC-Pytorch}\footnote{Unfortunately, we couldn't find its license from the website.}.

\subsection{GLOW} \label{Appendix B.2 : GLOW}
GLOW follows the existing configuration in \citep{kingma2018glow} except the number of flow modules.
The number of flow modules follows \citep{nalisnick2018deep, serra2019input}:
for training Fashion-MNIST, we use two blocks of 16 flows, and for CIFAR-10, we use three blocks of 8 flows with multi-scale.
For coupling blocks, 3-layer networks are used, and the number of hidden units is 200 for Fashion-MNIST and 400 for CIFAR-10.
The GLOWs are trained for 50 epochs (almost 100K steps) with batch size 64, learning rate 5e-4 using Adamax optimizer.
We also apply weight decay 5e-5 only for CIFAR-10.
Furthermore, we compute ROSE on the invertible 1x1 convolutional (InvConv) layers in GLOW, using the diagonal preconditioner for fast computation.
The comparison of applying InvConv layers and all convolutional layers is shown in Table \ref{table 7}.
The codes for training GLOW are obtained from \url{https://github.com/y0ast/Glow-PyTorch} (MIT License).

\subsection{Implementing Competing Methods} \label{Appendix B.3 : Implementing Competing Methods}

\paragraph{Likelihood Regret and Input Complexity}
For our comparison experiments (Table \ref{table 2}), LR(E) \citep{xiao2020likelihood} and IC \citep{serra2019input} follow the experimental settings mentioned in \citep{xiao2020likelihood}.
The implementation codes for LR(E) and IC are obtained from \url{https://github.com/XavierXiao/Likelihood-Regret} (MIT License). 

\paragraph{Likelihood Ratio}
Our main and background models used for LRatio \citep{ren2019likelihood} experiments have the same structure as those mentioned in Appendix \ref{Appendix B.1 : VAE}.
The implementation code for LRatio is obtained from \url{https://github.com/XavierXiao/Likelihood-Regret} (MIT License). 
\begin{itemize}
    \item Experiment in Section \ref{section 3.2} (Table \ref{table 1})\\
    For each perturbation, the implementation details are as following:
    \begin{itemize}
        \item `SVD' creates a blurred image by discarding the bottom $K$ nonzero singular values for each channel, as in \citep{choi2019novelty}.
        In Table \ref{table 1}, we set $K$ to 28, which performs better than 22.
        Comparison of $K = 22$ and $K = 28$ is presented in Table \ref{table 10}.
        \item `DCT' uses only high magnitude $K$ signals in the frequency domain to perturb images, as in \citep{choi2019novelty}, and we set this $K$ to 28.
        \item `GB' uses the kernel size of 3 and the standard deviation of 0.8.
        \item `Random' selects pixels with the probability of $\mu$ for a given image and replaces them with uniformly sampled random values between 0 and 255.
        In Table \ref{table 1}, we experiment with $\mu$ at 0.2, which performs better on MNIST than 0.3.
        Comparison of $\mu = 0.2$ and $\mu = 0.3$ is presented in Table \ref{table 10}.
    \end{itemize}
    Furthermore, \citep{ren2019likelihood} suggest $\lambda$ as another hyperparameter of $L_2$ regularization to improve performance.
    For all perturbations in Table \ref{table 1}, we use $\lambda = 0$, which performs better than $\lambda = 10$.
    
    \item Experiment in Section \ref{section 5.1} (Table \ref{table 2})\\
    For comparison with ROSE, we fix the perturbation `Random' for LRatio, and select the other configuration of the background models which produce the highest AUROC on MNIST dataset.
    \begin{itemize}
        \item For VAE trained on Fashion-MNIST, we set $\mu = 0.3$ and $\lambda = 0$.
        \item For VAE trained on CIFAR-10, we set $\mu = 0.2$ and $\lambda = 0$.
        \item For GLOW trained on Fashion-MNIST, we set $\mu = 0.3$ and $\lambda = 0.01$.
        \item For GLOW trained on CIFAR-10, we set $\mu = 0.2$ and $\lambda = 5e-5$.
    \end{itemize}
\end{itemize}

\paragraph{RND}
RND \citep{choi2019novelty} consists of a trainable predictor $f$ and randomly initialized target nets $g_0$ and $g_1$.
The target nets $g_0$ and $g_1$ are used to extract the representation of the in-distribution
and data with a specific perturbation applied to in-distribution, respectively.
Predictor $f$ is trained to be close to the representation of $g_0$ for in-distribution images and to be close to $g_1$ for the perturbed.
The OOD score in RND is defined as the $L_2$ distance of outputs of $f$ and $g_0$ when inputting test samples.
The model structures and training settings of RND follow the ones presented in \citep{choi2019novelty}.
The perturbations used in RND are SVD, DCT, GB, and Random with the same conditions as in LRatio.
The comparison of various perturbations on RND models is presented in Table \ref{table 11}.
The implementation code for RND is obtained from \url{https://github.com/sungikchoi/NVB} (MIT License).

\newpage
\section{Additional Quantitative Results} \label{Appendix C : Additional Results}
\subsection{Results for AUPRC and FPR80}
\begin{table}[H]
  \centering
  {\fontsize{8}{8}\selectfont  
  \begin{tabular}{cccccccccccc}
    \toprule 
    \multicolumn{1}{c}{} & \multicolumn{5}{c}{VAE} & \multicolumn{1}{c}{} & \multicolumn{5}{c}{GLOW}\\
    \cmidrule{2-7}
    \cmidrule{9-12}
    FMNIST & ROSE & LR$_{\text{E}}$ & NLL & LRatio & IC(png) & IC(jp2) & & ROSE & LRatio & IC(png) & IC(jp2)\\
    \midrule
    MNIST    & \textbf{0.996} & 0.968 & 0.334 & 0.857 & 0.948 & 0.494 & 
             & 0.600 & 0.905 & \textbf{0.942} & 0.663 \\
    KMNIST   & \textbf{0.997} & 0.991 & 0.622 & 0.963 & 0.666 & 0.436 & 
             & 0.547 & \textbf{0.954} & 0.486 & 0.367 \\
    NotMNIST & \textbf{1.000} & \textbf{1.000} & 0.980 & 0.875 & 0.851 & 0.946 & 
             & \textbf{0.876} & 0.490 & 0.479 & 0.465 \\
    Omniglot & \textbf{1.000} & \textbf{1.000} & \textbf{1.000} & 0.912 & \textbf{1.000} & \textbf{1.000} & 
             & \textbf{1.000} & 0.835 & \textbf{1.000} & \textbf{1.000} \\
    CIFAR-10 & 0.999 & 0.996 & \textbf{1.000} & 0.839 & 0.911 & 0.999 & 
             & 0.979 & 0.940 & 0.845 & \textbf{0.998} \\
    SVHN     & \textbf{1.000} & \textbf{1.000} & 0.999 & 0.512 & 0.998 & \textbf{1.000} & 
             & 0.953 & 0.332 & 0.799 & \textbf{0.999} \\
    Noise    & 0.995 & 0.998 & \textbf{1.000} & \textbf{1.000} & \textbf{1.000} & \textbf{1.000} & 
             & \textbf{1.000} & \textbf{1.000} & \textbf{1.000} & \textbf{1.000} \\
    Constant & \textbf{1.000} & 0.997 & 0.948 & 0.533 & \textbf{1.000} & \textbf{1.000} & 
             & 0.995 & 0.308 & \textbf{1.000} & \textbf{1.000} \\
    Overall  & \textbf{0.998} & 0.994 & 0.759 & 0.755 & 0.858 & 0.804 & 
             & \textbf{0.847} & 0.572 & 0.757 & 0.727 \\
    \bottomrule
    \toprule
    CIFAR-10 & ROSE & LR$_{\text{E}}$ & NLL & LRatio& IC(png) & IC(jp2) & & ROSE & LRatio & IC(png) & IC(jp2) \\
    \midrule
    SVHN     & 0.921 & \textbf{0.959} & 0.320 & 0.312 & 0.921 & 0.953 & 
             & 0.952 & 0.337 & 0.886 & \textbf{0.973} \\
    CelebA   & \textbf{0.798} & 0.723 & 0.487 & 0.410 & 0.587 & 0.417 & 
             & 0.720 & 0.611 & \textbf{0.746} & 0.442 \\
    LSUN     & 0.551 & 0.543 & 0.589 & 0.560 & \textbf{0.672} & 0.394 & 
             & 0.739 & 0.701 & \textbf{0.803} & 0.423 \\
    MNIST    & \textbf{1.000} & 0.997 & 0.307 & 0.841 & 0.991 & 0.989 & 
             & 0.999 & 0.307 & \textbf{1.000} & \textbf{1.000} \\
    FMNIST   & \textbf{0.996} & 0.991 & 0.310 & 0.732 & 0.995 & 0.994 & 
             & 0.992 & 0.307 & 0.998 & \textbf{1.000} \\
    NotMNIST & \textbf{0.997} & 0.996 & 0.310 & 0.494 & 0.992 & 0.994 & 
             & 0.995 & 0.307 & 0.996 & \textbf{0.998} \\
    Noise    & \textbf{1.000} & 0.998 & \textbf{1.000} & \textbf{1.000} & 0.420 & 0.396 & 
             & \textbf{1.000} & \textbf{1.000} & \textbf{1.000} & \textbf{1.000} \\
    Constant & 0.972 & 0.995 & 0.366 & 0.350 & \textbf{1.000} & \textbf{1.000} & 
             & \textbf{1.000} & 0.307 & \textbf{1.000} & \textbf{1.000} \\
    Overall  & \textbf{0.887} & 0.882 & 0.357 & 0.516 & 0.758 & 0.662 & 
             & 0.915 & 0.363 & \textbf{0.922} & 0.779 \\
    \bottomrule
  \end{tabular}
  }
  \vspace{2mm} 
  \caption{
  AUPRC of ROSE and other OOD detection metrics on the benchmark image datasets. 
  The top part of the table uses a VAE (\textbf{left}) and GLOW (\textbf{right}) model trained on Fashion-MNIST, and the bottom part of the table uses a VAE (\textbf{left}) and GLOW (\textbf{right}) model trained on CIFAR-10.
  To clarify the AUPRC, we define well-detected real OOD samples as True Negative (TN) while regarding well-detected real in-distribution samples as True Positive (TP).
  For all of our experiments, we employ test samples for in-distribution data (ex: CIFAR-10 test).
  }
  \label{table 4}
\end{table}


\begin{table}[H]
  \centering
  {\fontsize{8}{8}\selectfont  
  \begin{tabular}{cccccccccccc}
    \toprule 
    \multicolumn{1}{c}{} & \multicolumn{5}{c}{VAE} & \multicolumn{1}{c}{} & \multicolumn{5}{c}{GLOW}\\
    \cmidrule{2-7}
    \cmidrule{9-12}
    FMNIST & ROSE & LR$_{\text{E}}$ & NLL & LRatio & IC(png) & IC(jp2) & & ROSE & LRatio & IC(png) & IC(jp2)\\
    \midrule
    MNIST    & 0.007 & 0.056 & 0.975 & 0.246 & \textbf{0.000} & 0.805 & 
             & 0.664 & 0.199 & \textbf{0.129} & 0.590 \\
    KMNIST   & \textbf{0.003} & 0.007 & 0.680 & 0.036 & 0.547 & 0.873 & 
             & 0.741 & \textbf{0.099} & 0.787 & 0.995 \\
    NotMNIST & \textbf{0.000} & \textbf{0.000} & 0.011 & 0.001 & 0.201 & 0.011 & 
             & \textbf{0.195} & 0.800 & 0.920 & 0.936 \\
    Omniglot & \textbf{0.000} & \textbf{0.000} & \textbf{0.000} & \textbf{0.000} & \textbf{0.000} & \textbf{0.000} & 
             & \textbf{0.000} & 0.218 & \textbf{0.000} & \textbf{0.000} \\
    CIFAR-10 & 0.001 & 0.003 & \textbf{0.000} & 0.054 & 0.205 & \textbf{0.000} & 
             & 0.025 & \textbf{0.000} & 0.328 & 0.001 \\
    SVHN     & \textbf{0.000} & \textbf{0.000} & 0.002 & 0.983 & \textbf{0.000} & \textbf{0.000} & 
             & 0.062 & 1.000 & 0.409 & \textbf{0.000} \\
    Noise    & 0.012 & 0.004 & \textbf{0.000} & \textbf{0.000} & 0.594 & \textbf{0.000} & 
             & \textbf{0.000} & \textbf{0.000} & \textbf{0.000} & \textbf{0.000} \\
    Constant & \textbf{0.000} & \textbf{0.000} & 0.053 & 0.915 & \textbf{0.000} & \textbf{0.000} & 
             & \textbf{0.000} & 1.000 & \textbf{0.000} & \textbf{0.000} \\
    Overall  & \textbf{0.002} & 0.003 & 0.067 & 0.117 & 0.249 & 0.102 & 
             & \textbf{0.206} & 0.883 & 0.328 & 0.321 \\
    \bottomrule
    \toprule
    CIFAR-10 & ROSE & LR$_{\text{E}}$ & NLL & LRatio & IC(png) & IC(jp2) &
             & ROSE & LRatio & IC(png) & IC(jp2) \\
    \midrule
    SVHN     & 0.144 & 0.068 & 0.986 & 0.987 & 0.122 & \textbf{0.074} & 
             & 0.097 & 0.960 & 0.184 & \textbf{0.029} \\
    CelebA   & \textbf{0.373} & 0.508 & 0.802 & 0.886 & 0.678 & 0.867 & 
             & 0.511 & 0.679 & \textbf{0.483} & 0.846 \\
    LSUN     & 0.766 & 0.759 & 0.694 & 0.724 & \textbf{0.560} & 0.916 & 
             & 0.468 & 0.422 & \textbf{0.385} & 0.898 \\
    MNIST    & \textbf{0.001} & 0.004 & 1.000 & 0.350 & 0.019 & 0.023 & 
             & 0.003 & 1.000 & 0.001 & \textbf{0.000} \\
    FMNIST   & \textbf{0.004} & 0.011 & 0.998 & 0.442 & 0.013 & 0.013 & 
             & 0.018 & 1.000 & 0.004 & \textbf{0.000} \\
    NotMNIST & \textbf{0.001} & 0.004 & 0.998 & 0.851 & 0.020 & 0.013 & 
             & 0.014 & 1.000 & 0.008 & \textbf{0.005} \\
    Noise    & \textbf{0.000} & 0.004 & \textbf{0.000} & \textbf{0.000} & 0.870 & 0.901 & 
             & \textbf{0.000} & \textbf{0.000} & \textbf{0.000} & \textbf{0.000} \\
    Constant & 0.019 & 0.002 & 0.988 & 0.996 & \textbf{0.000} & \textbf{0.000} & 
             & \textbf{0.000} & 1.000 & \textbf{0.000} & \textbf{0.000} \\
    Overall  & \textbf{0.080} & 0.081 & 0.999 & 0.875 & 0.328 & 0.725 & 
             & \textbf{0.072} & 1.000 & 0.092 & \textbf{0.072} \\
    \bottomrule
  \end{tabular}
  }
  \vspace{2mm} 
  \caption{
  FPR80 of ROSE and other OOD detection metrics on the benchmark image datasets. 
  The top part of the table uses a VAE (\textbf{left}) and GLOW (\textbf{right}) model trained on Fashion-MNIST, and the bottom part of the table uses a VAE (\textbf{left}) and GLOW (\textbf{right}) model trained on CIFAR-10.
  To clarify the FPR80, we define TN and TP in the same way as in Table \ref{table 4}.
  For all of our experiments, we employ test samples for in-distribution data (ex: CIFAR-10 test).
  }
  \label{table 5}
\end{table}

\subsection{DIAG vs. EKFAC}

We introduce the diagonal preconditioner (DIAG) and EKFAC as two approximation methods of computing the inverse of the Fisher information matrix.
As shown in Table \ref{table 6}, the overall performance of EKFAC is slightly superior to DIAG, but the difference is tiny.
This fact indicates ROSE on VAE is also robust to preconditioning.

\begin{table}[H]
  \centering
  {\fontsize{8}{9}\selectfont  
  \begin{tabular}{ccccccccccc}
    \toprule 
    FMNIST & MNIST & KMNIST & NotMNIST & CIFAR-10 & SVHN & Noise & Constant & Overall\\
    \midrule
    DIAG   & 0.971 & 0.978 & 0.997 & 0.992 & 0.997 & 0.885 & 0.997 & 0.980 \\
    EKFAC  & 0.995 & 0.997 & 1.000 & 0.999 & 1.000 & 0.992 & 1.000 & 0.998 \\ 
    \bottomrule
    \toprule
    CIFAR-10 & SVHN & CelebA & LSUN & MNIST & FMNIST & Noise & Constant & Overall\\
    \midrule
    DIAG   & 0.877 & 0.713 & 0.512 & 0.997 & 0.962 & 0.911 & 0.957 & 0.889 \\
    EKFAC  & 0.906 & 0.778 & 0.526 & 0.999 & 0.996 & 1.000 & 0.973 & 0.918 \\
    
    \bottomrule
  \end{tabular}
  }
  \vspace{2mm} 
  \caption{
  AUROC of ROSE on VAE with Diagonal preconditioner (DIAG) vs. EKFAC preconditioner.
  For both DIAG and EKFAC, we utilize all convolutional layers in the encoder.
  For all situations, EKFAC outperforms DIAG.
  }
  \label{table 6}
\end{table}

\subsection{Addtional results for GLOW}

As mentioned in Appendix \ref{Appendix B.2 : GLOW}, ROSE on GLOW is measured using only invertible convolutional layers.
Results using all layers of GLOW are shown in Table \ref{table 7}.

\begin{table}[H]
  \centering
  {\fontsize{8}{9}\selectfont  
  \begin{tabular}{ccccccccccc}
    \toprule 
    FMNIST & MNIST & KMNIST & NotMNIST & CIFAR-10 & SVHN & Noise & Constant & Overall\\
    \midrule
    InvConv & 0.562 & 0.518 & 0.896 & 0.976 & 0.951 & 1.000 & 0.996 & 0.885 \\
    Full    & 0.733 & 0.787 & 0.829 & 0.858 & 0.835 & 1.000 & 0.852 & 0.857 \\ 
    \bottomrule
    \toprule
    CIFAR-10 & SVHN & CelebA & LSUN & MNIST & FMNIST & Noise & Constant & Overall\\
    \midrule
    InvConv & 0.940 & 0.694 & 0.730 & 0.998 & 0.988 & 1.000 & 1.000 & 0.932 \\
    Full    & 0.827 & 0.572 & 0.669 & 0.802 & 0.798 & 0.999 & 0.952 & 0.809 \\
    
    \bottomrule
  \end{tabular}
  }
  \vspace{2mm} 
  \caption{
  AUROC of ROSE on GLOW with Invertible 1$\times$1 convolutional layers (InvConv) vs. All convolutional layers (Full).
  Explicitly, `InvConv' indicates computing ROSE at only InvConv layers, and `Full' indicates computing ROSE at all convolutional layers, including InvConv layers and coupling block layers.
  Note that each flow module has one InvConv layer and one coupling block consisting of three convolutional layers \citep{kingma2018glow, nalisnick2018deep}.
  The overall performance of InvConv is slightly superior to Full.
  For Fashion-MNIST, Full is better than InvConv on MNIST/KMNIST only and worse than InvConv on all others.
  For CIFAR-10, InvConv is better than Full on all datasets.
  }
  \label{table 7}
\end{table}

\subsection{Results of Likelihood Ratio Test and SVD-RND}

\begin{table}[H]
    \setlength{\tabcolsep}{2.5pt}
    \centering
    {\fontsize{9}{10}\selectfont  
    \aboverulesep=0ex 
    \belowrulesep=0ex 
    \begin{tabular}{llccccccccc|cc}
        \toprule
        \multicolumn{2}{c}{Trained on CIFAR-10} & \multicolumn{9}{c}{Brightness of SVHN} & \multirow{2}{*}{Mean} & \multirow{2}{*}{Std} \\
        \cmidrule(lr){3-11}
        Method & Perturbation& 0.2$\times$ & 0.4$\times$ & 0.6$\times$ & 0.8$\times$ & 1$\times$ & 1.2$\times$ & 1.4$\times$ & 1.6$\times$ & \multicolumn{1}{c}{1.8$\times$} & &\\
        \midrule
        \multirow{4}{*}{LRatio}
            & SVD & 0.296 & 0.726 & 0.874 & 0.898 & 0.913 & 0.910 & 0.880 & 0.825 & 0.749 &0.786&0.185 \\
            & DCT & 0.653 & 0.937 & 0.963 & 0.956 & 0.938 & 0.937 & 0.927 & 0.910 & 0.886 &0.901&0.090\\
            & GB & 0.995 & 0.991 & 0.982 & 0.976 & 0.947 & 0.802 & 0.638 & 0.453 & 0.312 & 0.788 & 0.246 \\
            & Random & 0.004 & 0.019 & 0.034 & 0.043 & 0.050 & 0.113 & 0.208 & 0.326 & 0.457 &0.139&0.150\\
        \midrule
        
        \multirow{4}{*}{RND}
            & SVD & 0.993 & 0.979 & 0.973 & 0.973 & 0.978 & 0.976 & 0.966 & 0.957 & 0.944 &0.971&0.032\\
            & DCT & 0.998 & 0.981 & 0.952 & 0.922 & 0.907 & 0.889 & 0.875 & 0.867 & 0.860 &0.917&0.047\\
            & GB & 0.991 & 0.971 & 0.953 & 0.942 & 0.939 & 0.932 & 0.926 & 0.916 & 0.907 & 0.942 & 0.025 \\
            & Random & 0.006 & 0.009 & 0.025 & 0.060 & 0.147 & 0.256 & 0.360 & 0.443 & 0.496 & 0.200 & 0.183 \\
        \midrule
        ROSE && 0.809 & 0.807 & 0.861 & 0.893 & 0.907 & 0.920 & 0.934 & 0.951 & 0.965 & 0.894 & 0.054\\
    
        \bottomrule
    \end{tabular}
    }
    \vspace{2mm}
    \caption{
    AUROC of LRatio \citep{ren2019likelihood}, RND \citep{choi2019novelty}, and ROSE(ours) on SVHN with varying brightness. All models are trained on CIFAR-10. Mean stands for the mean of AUROC in a row, and Std indicates the standard deviation of AUROC. The settings of perturbations are the same as in Table \ref{table 1}. See \ref{Appendix B.3 : Implementing Competing Methods}.
    }
    \label{table 8}
\end{table}

\begin{table}[H]
    \setlength{\tabcolsep}{2.5pt}
    \centering
    {\fontsize{9}{10}\selectfont  
    \aboverulesep=0ex 
    \belowrulesep=0ex 
    \begin{tabular}{llccccccccc|cc}
        \toprule
        \multicolumn{2}{c}{Trained on CIFAR-10} & \multicolumn{9}{c}{Brightness of Omniglot} & \multirow{2}{*}{Mean} & \multirow{2}{*}{Std} \\
        \cmidrule(lr){3-11}
        Method & Perturbation& 0.2$\times$ & 0.4$\times$ & 0.6$\times$ & 0.8$\times$ & 1$\times$ & 1.2$\times$ & 1.4$\times$ & 1.6$\times$ & \multicolumn{1}{c}{1.8$\times$} & &\\
        \midrule
        \multirow{4}{*}{LRatio}
            & SVD & 0.046 & 0.463 & 0.192 & 0.005 & 0.001 & 0.003 & 0.007 & 0.010 & 0.011 & 0.082 & 0.146\\
            & DCT & 0.460 & 0.036 & 0.005 & 0.000 & 0.033 & 0.021 & 0.106 & 0.016 & 0.015 & 0.077 & 0.139\\
            & GB & 0.122 & 0.161 & 0.126 & 0.009 & 0.001 & 0.001 & 0.001 & 0.001 & 0.001 & 0.047 & 0.064 \\
            & Random & 0.010 & 0.157 & 0.338 & 0.008 & 0.961 & 0.997 & 0.999 & 0.909 & 0.999 & 0.598 & 0.430\\
        \midrule
        
        \multirow{4}{*}{RND}
            & SVD & 0.221 & 0.347 & 0.499 & 0.781 & 0.957 & 0.961 & 0.964 & 0.966 & 0.968 & 0.740 & 0.285\\
            & DCT & 0.069 & 0.279 & 0.477 & 0.731 & 0.913 & 0.935 & 0.939 & 0.938 & 0.936 & 0.691 & 0.315\\
            & GB & 0.043 & 0.345 & 0.476 & 0.778 & 0.950 & 0.968 & 0.973 & 0.974 & 0.974 & 0.720 & 0.328 \\
            & Random & 0.013 & 0.197 & 0.451 & 0.788 & 0.970 & 0.981 & 0.983 & 0.984 & 0.984 & 0.706 & 0.363 \\
        \midrule
        ROSE && 0.893 & 0.803 & 0.898 & 0.819 & 0.996 & 0.999 & 0.999 & 0.999 & 0.999 & 0.934 & 0.078 \\
    
        \bottomrule
    \end{tabular}
    }
    \vspace{2mm}
    \caption{
    AUROC of LRatio \citep{ren2019likelihood}, RND \citep{choi2019novelty} and ROSE(ours) on Omniglot with varying brightness. All models are trained on CIFAR-10. Mean stands for the mean of AUROC in a row, and Std indicates the standard deviation of AUROC. The settings of perturbations are the same as in Table \ref{table 1} except that $K$ is 22 for SVD in this Table. See \ref{Appendix B.3 : Implementing Competing Methods}.
    }
    \label{table 9}
\end{table}

\begin{table}[H]
    \centering
    {\fontsize{9}{10}\selectfont  
    \begin{tabular}{ccccccccc}
    \toprule 
        CIFAR-10 & SVD22 & SVD28 & DCT28 & GB & Random0.1 & Random0.2 & Random0.3  \\
    \midrule
        SVHN     & 0.564 & 0.913 & 0.938 & 0.947 & 0.040 & 0.050 & 0.071\\
        CelebA   & 0.343 & 0.324 & 0.373 & 0.458 & 0.330 & 0.333 & 0.340\\
        LSUN     & 0.514 & 0.581 & 0.573 & 0.485 & 0.597 & 0.576 & 0.580\\
        MNIST    & 0.984 & 0.731 & 0.157 & 0.117 & 0.317 & 0.825 & 0.788\\
        FMNIST   & 0.832 & 0.655 & 0.348 & 0.230 & 0.514 & 0.783 & 0.716\\
        NotMNIST & 0.339 & 0.163 & 0.154 & 0.092 & 0.253 & 0.506 & 0.533\\
        Noise    & 0.225 & 0.906 & 0.981 & 0.839 & 1.000 & 1.000 & 1.000\\
        Constant & 0.328 & 0.418 & 0.928 & 0.713 & 0.238 & 0.223 & 0.235\\
        Overall  & 0.514 & 0.536 & 0.456 & 0.407 & 0.379 & 0.610 & 0.615\\ 
    \bottomrule
    \end{tabular}
    }
    \vspace{2mm} 
    \caption{
    AUROC of LRatio \citep{ren2019likelihood} on CIFAR-10 VAE according to various perturbation methods.
    Background models are trained with $\lambda = 0$.
    }
    \label{table 10}
\end{table}

\begin{table}[H]
  \centering
  {\fontsize{9}{10}\selectfont  
  \begin{tabular}{ccccccccc}
    \toprule 
    CIFAR-10 & SVD22 & SVD28 & DCT28 & GB & Random0.1 & Random0.2 & Random0.3 \\
    \midrule
    SVHN    & 0.984 & 0.978 & 0.907 & 0.939 & 0.143 & 0.147 & 0.146\\
    CelebA  & 0.582 & 0.566 & 0.548 & 0.554 & 0.573 & 0.577 & 0.573\\
    LSUN & 0.562 & 0.563 & 0.553 & 0.564 & 0.570 & 0.573 & 0.570\\
    MNIST & 0.933 & 0.937 & 0.897 & 0.961 & 0.928 & 0.933 &  0.925\\
    FMNIST & 0.927 & 0.824 & 0.794 & 0.841 & 0.784 & 0.788 & 0.789  \\
    NotMNIST & 0.978 & 0.979 & 0.965 & 0.985 & 0.960 & 0.964 & 0.961 \\
    Noise    & 0.999 & 0.999 & 1.000 & 1.000 & 1.000 & 1.000 & 1.000 \\
    Constant & 0.999 & 0.999 & 0.999 & 1.000 & 0.825 & 0.828 & 0.822 \\
    Overall  & 0.884 & 0.866 & 0.846 & 0.871 & 0.773 & 0.777 & 0.774 \\ 
    \bottomrule
  \end{tabular}
  }
  \vspace{2mm} 
  \caption{AUROC of RND \citep{choi2019novelty} models on CIFAR-10 according to various perturbation methods.}
  \label{table 11}
\end{table}

\newpage
\section{Parameter Selection and Transfer Learning} \label{Appendix D : Parameter Selection}
\begin{figure}[H]
\begin{subfigure}[]{0.330\textwidth}
  \centering
  \includegraphics[width=1\textwidth]{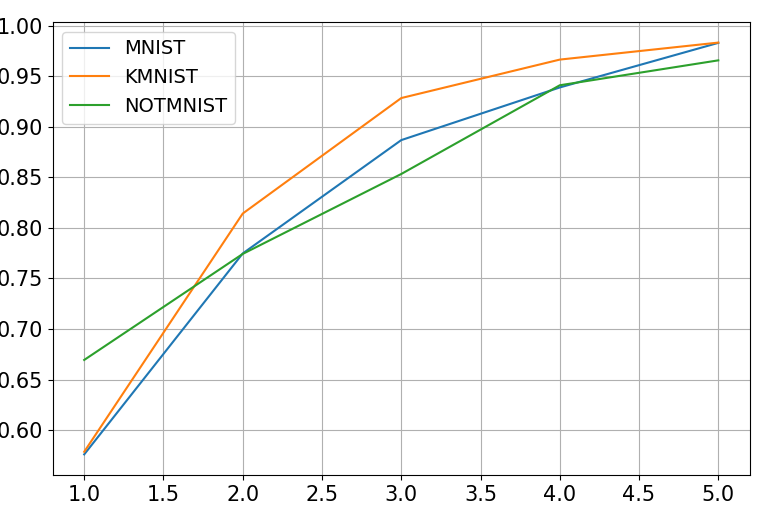}
  \caption{}
\end{subfigure}
\begin{subfigure}[]{0.330\textwidth}
  \centering
  \includegraphics[width=1\textwidth]{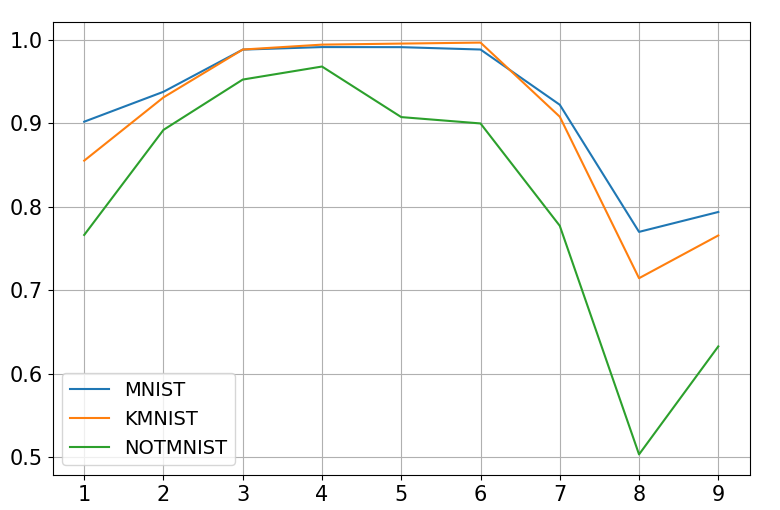}
  \caption{}
\end{subfigure}
\begin{subfigure}[]{0.330\textwidth}
  \centering
  \includegraphics[width=1\textwidth]{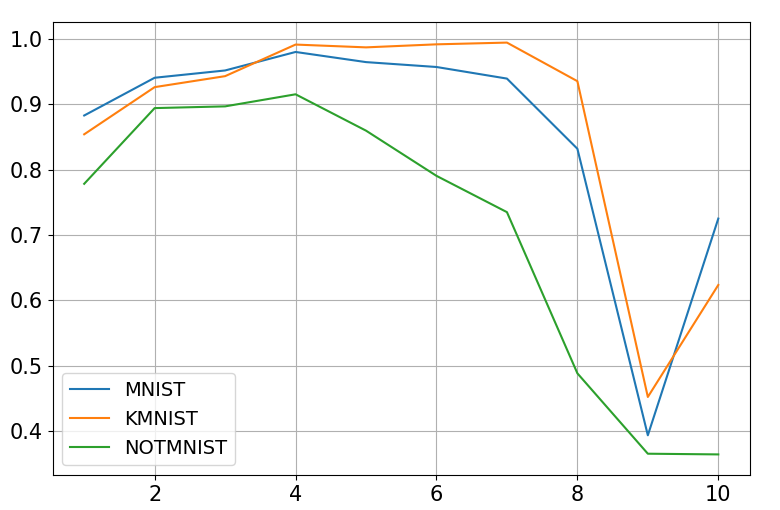}
  \caption{}
\end{subfigure}
\begin{subfigure}[]{0.330\textwidth}
  \centering
  \includegraphics[width=1\textwidth]{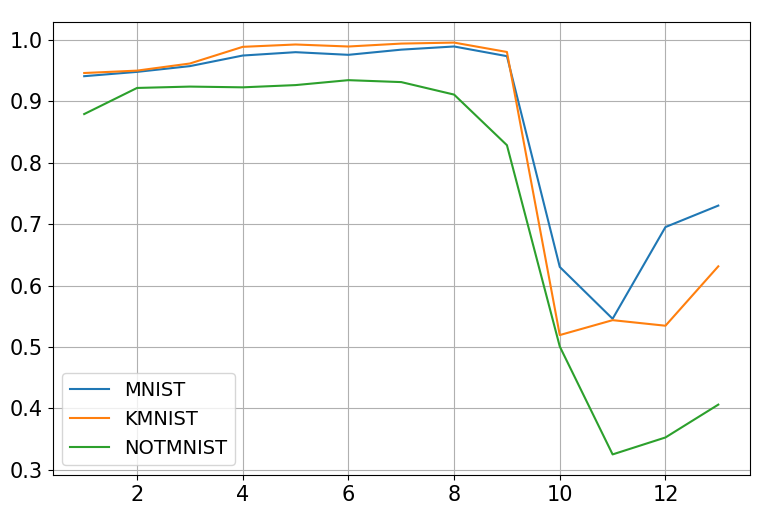}
  \caption{}
\end{subfigure}
\begin{subfigure}[]{0.330\textwidth}
  \centering
  \includegraphics[width=1\textwidth]{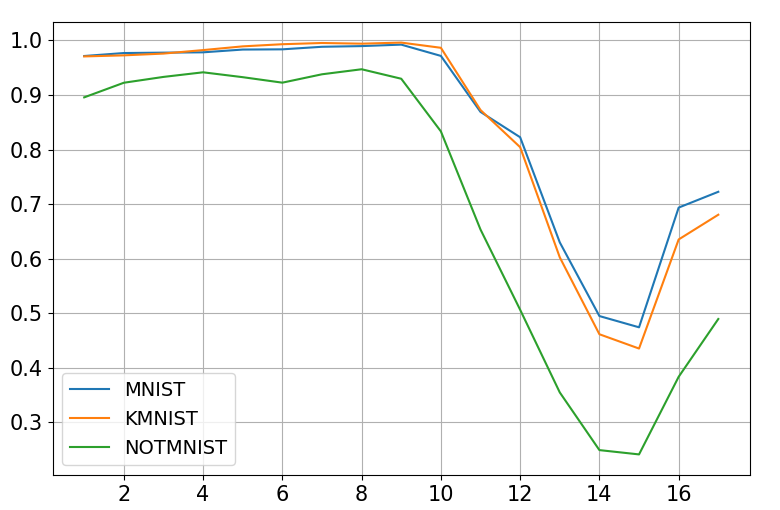}
  \caption{}
\end{subfigure}
\begin{subfigure}[]{0.330\textwidth}
  \centering
  \includegraphics[width=1\textwidth]{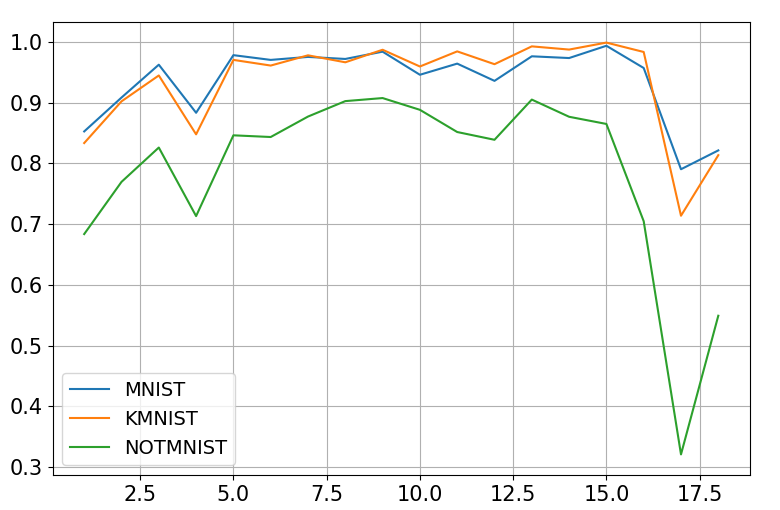}
  \caption{}
\end{subfigure}
\caption{
Layer-wise AUROC on various classifiers trained on Fashion-MNIST;
(a) CNN architecture with four convolutional layers,
(b) VGG-11,
(c) VGG-13,
(d) VGG-16,
(e) VGG-19,
(f) ResNet-18.
All network has one fully-connected layer.
The horizontal axis indicates the depth of each layer and the vertical axis indicates the AUROC of each layer.
}
\label{figure 6}
\end{figure}

As reported in \citep{havtorn2021hierarchical}, the data representations extracted from the early layers of two deep hierarchical VAEs, each trained on different datasets, are highly correlated.
This observation implies that early layers in VAE extract low-level features which generalize well across the data.
Such low-level features dominate the likelihood of a sample, which leads to failure in OOD detection \citep{havtorn2021hierarchical}.
Similarly, the likelihood in the latent space of a normalizing flow is assigned properly (i.e., the low likelihood for OOD and the high likelihood for in-distribution); however, an opposite phenomenon is observed for the input likelihood \citep{nalisnick2018deep}.
This is because the volume element of the normalizing flow, which is induced from shallow layers, dominates the likelihood.
Moreover, several studies observed that as the layers go deeper, more likely the features to obtain semantic properties \citep{gatys2016image}.
To summarize, various studies show that the features of early layers represent local properties and exert influence on the likelihood estimation, hindering OOD detection.

Our score can be seen as an ensemble of layer-wise scores.
From this perspective, we carefully predict that our score may be a new measurement of the extent to which semantic properties are extracted by each layer.
If the layer tends to extract only local properties, that layer may fail to detect an OOD sample of which texture is similar to in-distribution data samples.
To test such potential, we train classifiers such as a basic convolutional neural network (basic CNN) derived from AlexNet \citep{NIPS2012_c399862d}, VGGNet \citep{vgg}, ResNet \citep{He_2016_CVPR}, and observe how well each layer of these neural networks can detect anomalies.

Six classifiers on Fashion-MNIST are trained; Basic CNN, VGG-11, VGG-13, VGG-16, VGG-19, and ResNet-18.
Basic CNN consists of four convolutional layers.
For all networks, we use one fully connected layer.
For test samples, the arguments of the maxima (argmax) of classification probability are used as a label.
Furthermore, the cross entropy loss is applied both for training and evaluation.
AUROC is used for the detection metric. 
MNIST, KMNIST, and NotMNIST are used as the OOD datasets since these data have the most similar local properties to Fashion-MNIST among our test datasets.
Moreover, the diagonal preconditioner is employed.

We predict that in shallow networks, such as basic CNN, semantic features will be extracted from relatively deep layers.
Hence, AUROCs are expected to be higher as the layer deepens.
On the other hand, in the deep networks, we conjecture that the semantic part may be slowly pulled out at the middle layer.
In particular, Fashion-MNIST is a simple dataset so that even VGG-11 can be an over-weighted model.

Figure \ref{figure 6} shows our results on six classifiers.
As we expected, AUROCs increase in basic CNN as the layer goes deeper. 
Moreover, AUROCs of VGG series and ResNet-18 increase gradually and drop off at the deep layers.
Through these observations, we expect that our score (on each layer) can be a key to identify the layer's properties.
Further experiments and discussions remain as future work.

\section{Proofs} \label{Appendix E : proofs}
\begin{prop}
    Let $\theta_0$ be the marginal maximum likelihood on uninformative uniform prior.\footnote{Note that our theorem can be extended to Gaussian or Laplacian prior to consider weight decay.}
    Then, for certain regularization condition, score $S_X(\Tilde{x})$ is approximated as
    \begin{equation}
        \mathcal{S}_X(\Tilde{x})
        = \nabla_\theta \log{p(\Tilde{x}|\theta_0)}^T \Delta\theta
        +\frac{1}{2}\Delta\theta^T (\nabla_\theta^2 \log{p(\Tilde{x}|\theta_0)}-N \mathcal{F}_X
        (\theta_0))\Delta\theta +O(\lVert \Delta\theta\rVert^3),
    \end{equation}
    where $\Delta\theta = \theta_1-\theta_0$, $N=|X|$, and $\mathcal{F}_X(\theta_0)=\mathbb{E}_{x\in X}[\nabla_\theta \log{p(x|\theta_0)} \nabla_\theta \log{p(x|\theta_0)}^T]$.
\end{prop}
\begin{proof}
Let $\Delta \theta = \theta_1 - \theta_0$. By applying the Bayes' theorem, we obtain
\begin{align}
    \mathcal{S}_X(\Tilde{x}) 
    &= \log{p(\theta_1|X, \Tilde{x})} - \log{p(\theta_0|X, \Tilde{x})}\\
    &=  \log{p(\Tilde{x}|X, \theta_1)}-\log{p(\Tilde{x}|X, \theta_0)}+\log{p(\theta_1|X)}-\log{p(\theta_0|X)}.
\end{align}
Then, by the second-order approximation,
\begin{align}
\mathcal{S}_X(\Tilde{x})
    &= \nabla_\theta \log{p(\Tilde{x}|\theta_0)}^T \Delta \theta + \frac{1}{2}\Delta \theta^T \big(\nabla_\theta^2 \log{p(\Tilde{x}|\theta_0)} + \nabla_\theta^2 \log{p(\theta_0|X)}\big) \Delta \theta + O(\lVert \Delta \theta \rVert^3)\\
    &= \nabla_\theta \log{p(\Tilde{x}|\theta_0)}^T \Delta \theta + \frac{1}{2}\Delta \theta^T \nabla_\theta^2 \log{p(\Tilde{x}|\theta_0)}\Delta \theta - \frac{N}{2}\Delta\theta^T \mathcal{F}_X(\theta_0) \Delta\theta+ O(\lVert \Delta \theta \rVert^3)
\end{align}

\end{proof}

\newpage
\section{Images} \label{Appendix F : Images}
\subsection{Dataset Explanation}

\begin{figure}[H]
\begin{subfigure}[]{0.5\textwidth}
  \centering
  \includegraphics[width=1\textwidth]{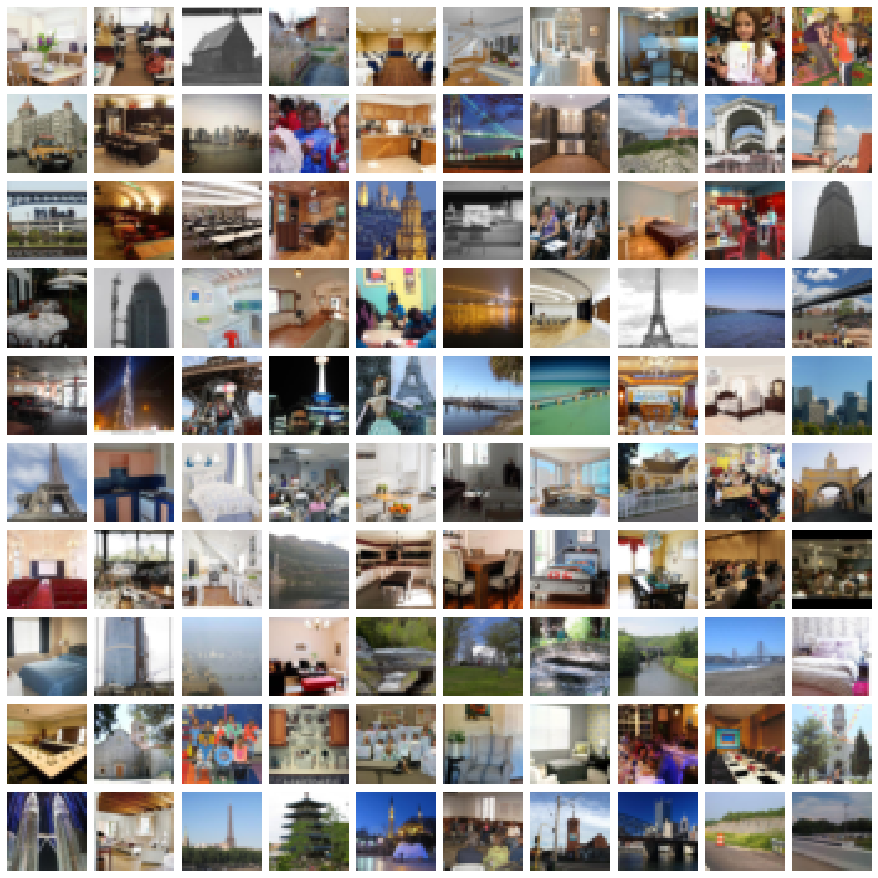}
  \caption{}
\end{subfigure}
\begin{subfigure}[]{0.5\textwidth}
  \centering
  \includegraphics[width=1\textwidth]{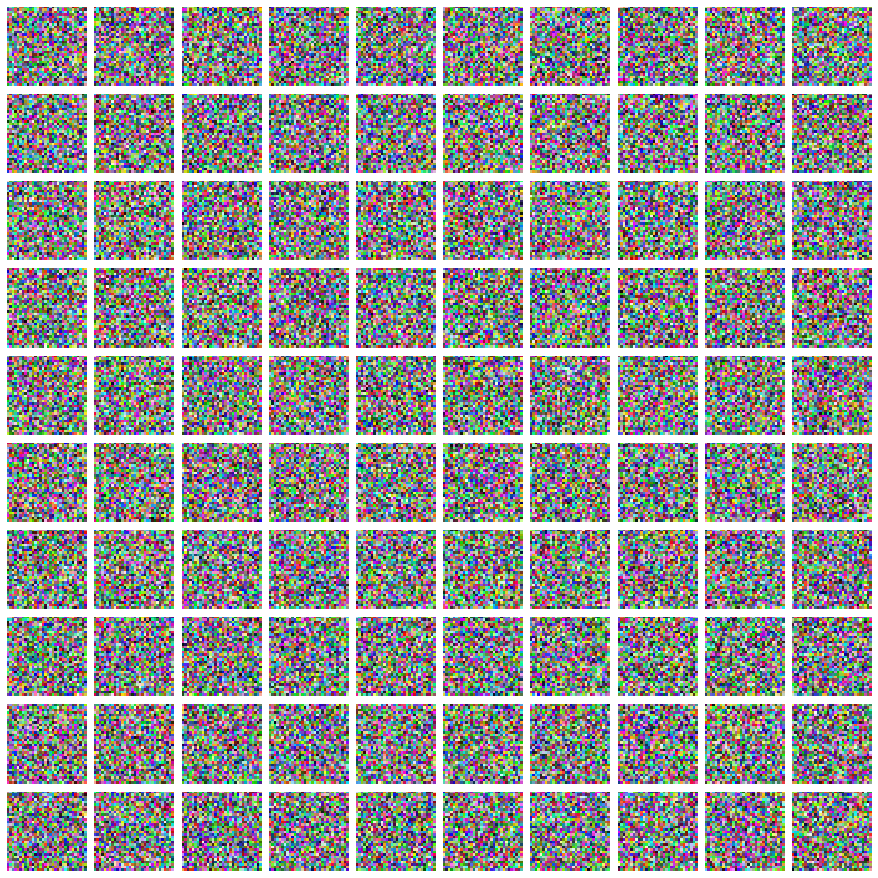}
  \caption{}
\end{subfigure}

\caption{
(a) shows some examples of LSUN, which are resized and cropped by 32 $\times$ 32.
(b) shows generated noise samples of which each pixel is randomly sampled from $\{0, 1, \cdots, 255\}$.
}
\label{figure dataset}
\end{figure}

\subsection{VAE Reconstruction Examples}

\begin{figure}[H]
\begin{subfigure}[]{0.5\textwidth}
  \centering
  \includegraphics[width=1\textwidth]{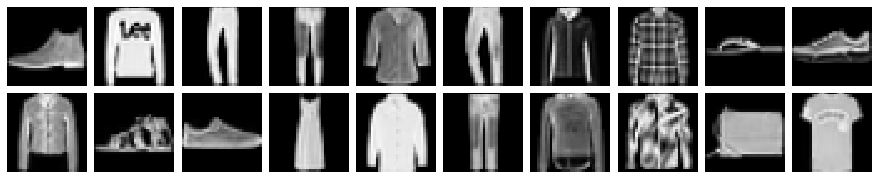}
\end{subfigure}
\begin{subfigure}[]{0.5\textwidth}
  \centering
  \includegraphics[width=1\textwidth]{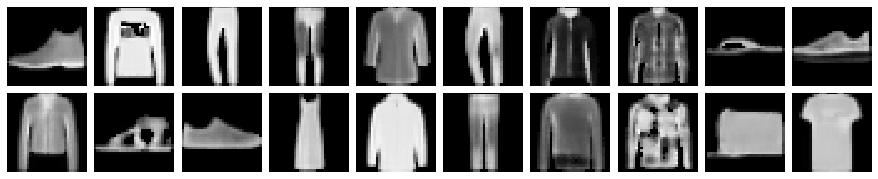}
\end{subfigure}

\begin{subfigure}[]{0.5\textwidth}
  \centering
  \includegraphics[width=1\textwidth]{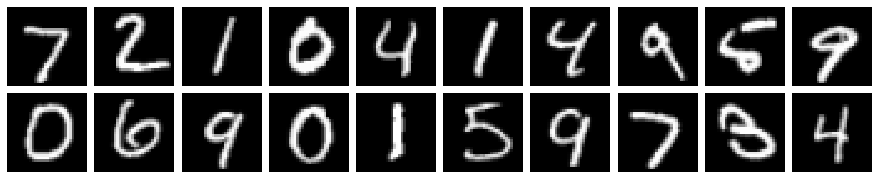}
  \caption{}
\end{subfigure}
\begin{subfigure}[]{0.5\textwidth}
  \centering
  \includegraphics[width=1\textwidth]{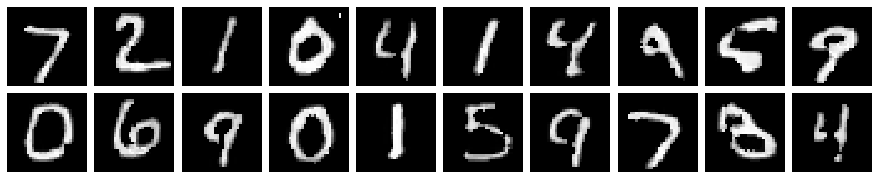}
  \caption{}
\end{subfigure}

\begin{subfigure}[]{0.5\textwidth}
  \centering
  \includegraphics[width=1\textwidth]{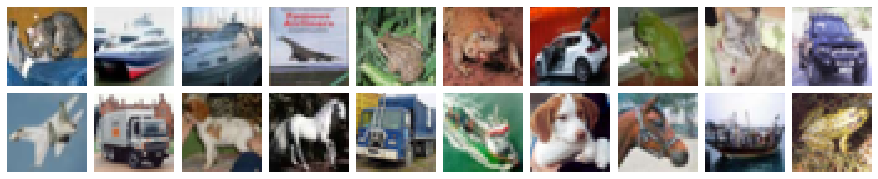}
\end{subfigure}
\begin{subfigure}[]{0.5\textwidth}
  \centering
  \includegraphics[width=1\textwidth]{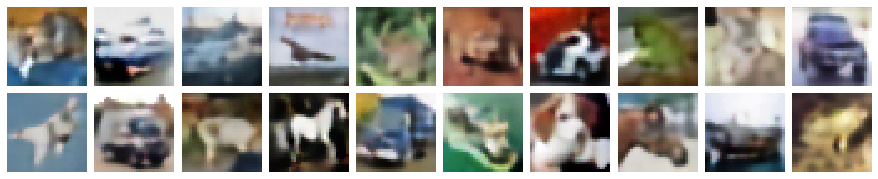}
\end{subfigure}

\begin{subfigure}[]{0.5\textwidth}
  \centering
  \includegraphics[width=1\textwidth]{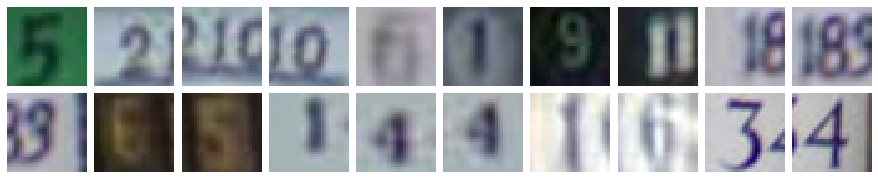}
  \caption{}
\end{subfigure}
\begin{subfigure}[]{0.5\textwidth}
  \centering
  \includegraphics[width=1\textwidth]{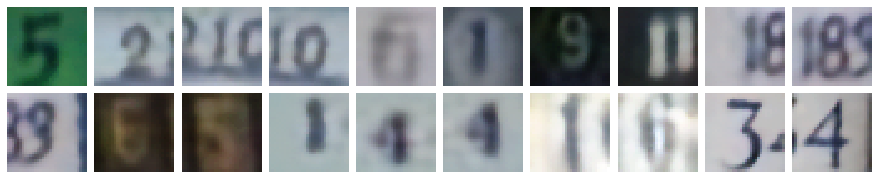}
  \caption{}
\end{subfigure}

\caption{
(a) shows some examples of Fashion-MNIST (upper) and MNIST (lower).
(b) shows reconstruction of (a) by VAE trained on Fashion-MNIST.
(c) shows some examples of CIFAR-10 (upper) and SVHN (lower).
(d) shows reconstruction of (c) by VAE trained on CIFAR-10.
Though MNIST samples in (a) and SVHN samples in (c) are definite outliers, they are reconstructed in high quality.
}
\end{figure}

\subsection{GLOW Reconstruction Examples}
\begin{figure}[H]
\begin{subfigure}[]{0.5\textwidth}
  \centering
  \includegraphics[width=1\textwidth]{images/appendixF/fmnist.png}
\end{subfigure}
\begin{subfigure}[]{0.5\textwidth}
  \centering
  \includegraphics[width=1\textwidth]{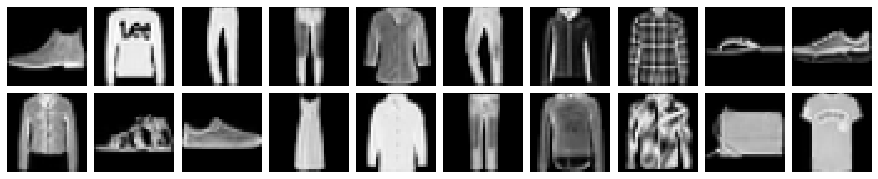}
\end{subfigure}

\begin{subfigure}[]{0.5\textwidth}
  \centering
  \includegraphics[width=1\textwidth]{images/appendixF/mnist.png}
  \caption{}
\end{subfigure}
\begin{subfigure}[]{0.5\textwidth}
  \centering
  \includegraphics[width=1\textwidth]{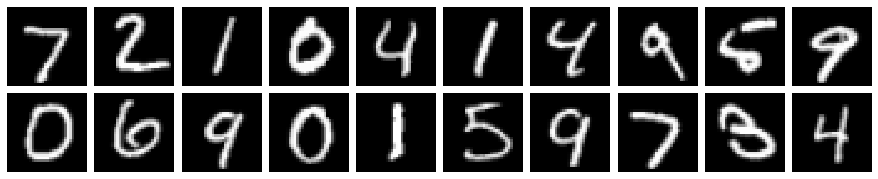}
  \caption{}
\end{subfigure}

\begin{subfigure}[]{0.5\textwidth}
  \centering
  \includegraphics[width=1\textwidth]{images/appendixF/cifar10.png}
\end{subfigure}
\begin{subfigure}[]{0.5\textwidth}
  \centering
  \includegraphics[width=1\textwidth]{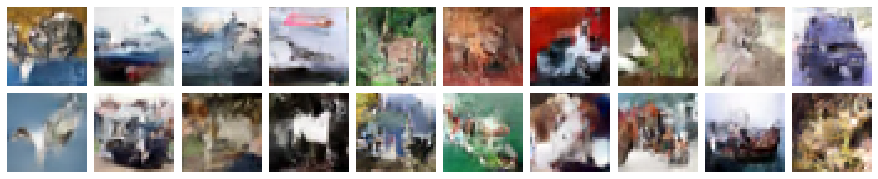}
\end{subfigure}

\begin{subfigure}[]{0.5\textwidth}
  \centering
  \includegraphics[width=1\textwidth]{images/appendixF/svhn.png}
  \caption{}
\end{subfigure}
\begin{subfigure}[]{0.5\textwidth}
  \centering
  \includegraphics[width=1\textwidth]{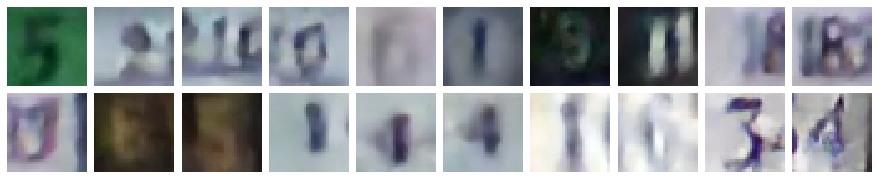}
  \caption{}
\end{subfigure}

\caption{
(a) shows some examples of Fashion-MNIST (upper) and MNIST (lower).
(b) shows reconstruction of (a) by GLOW trained on Fashion-MNIST.
(c) shows some examples of CIFAR-10 (upper) and SVHN (lower).
(d) shows reconstruction of (c) by GLOW trained on CIFAR-10.
Though MNIST samples in (a) and SVHN samples in (c) are definite outliers, they are reconstructed in high quality.
}
\end{figure}

\subsection{Brightness Examples}
\begin{figure}[H]
\centering
\includegraphics[width=1\textwidth]{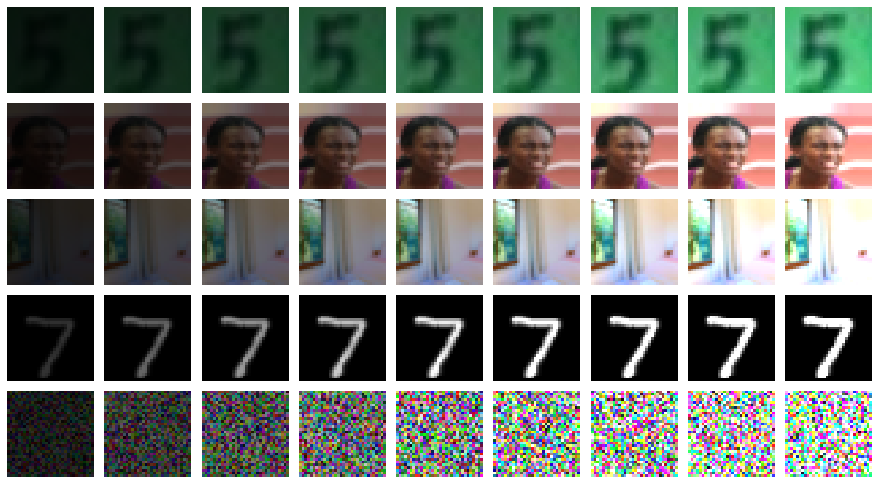}

\caption{ 
Examples with changing brightness.
From left to right, brightness changes from 0.2x to 1.8x at 0.2 intervals of the original image.
From top to down, the data sample is drawn from SVHN, CelebA, LSUN, MNIST, and noise.
}
\end{figure}

\subsection{Blurring Examples}

\begin{figure}[H]
\centering
\begin{subfigure}[]{0.100\textwidth}
  \centering
  \includegraphics[width=1\textwidth]{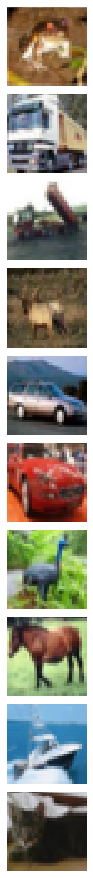}
  \caption{}
\end{subfigure}
\begin{subfigure}[]{0.100\textwidth}
  \centering
  \includegraphics[width=1\textwidth]{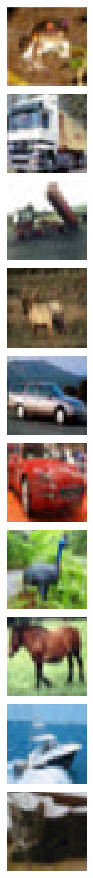}
  \caption{}
\end{subfigure}
\begin{subfigure}[]{0.100\textwidth}
  \centering
  \includegraphics[width=1\textwidth]{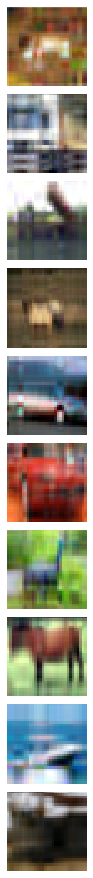}
  \caption{}
\end{subfigure}
\begin{subfigure}[]{0.100\textwidth}
  \centering
  \includegraphics[width=1\textwidth]{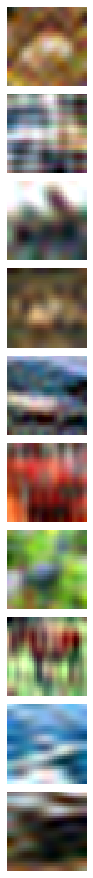}
  \caption{}
\end{subfigure}
\begin{subfigure}[]{0.100\textwidth}
  \centering
  \includegraphics[width=1\textwidth]{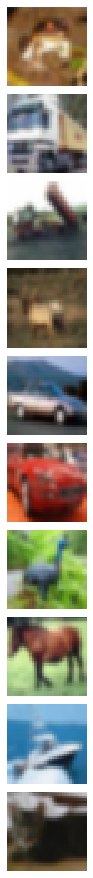}
  \caption{}
\end{subfigure}
\begin{subfigure}[]{0.100\textwidth}
  \centering
  \includegraphics[width=1\textwidth]{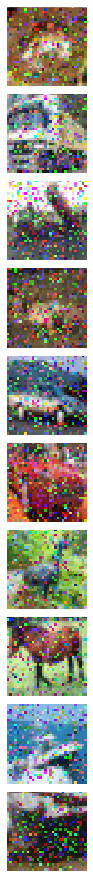}
  \caption{}
\end{subfigure}
\begin{subfigure}[]{0.100\textwidth}
  \centering
  \includegraphics[width=1\textwidth]{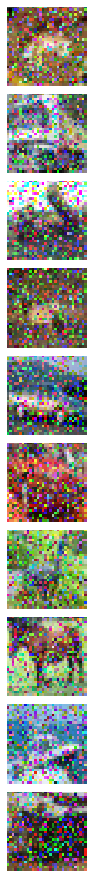}
  \caption{}
\end{subfigure}
\caption{ 
Visualization of various perturbation methods.
(b)-(g) are perturbed images of (a) original CIFAR-10 training samples.
(b) is the SVD blurring by discarding 22 singular values and (c) discards 28 singular values.
(d) shows DCT blurred images and (e) shows samples after Gaussian blurring.
(f) and (g) are examples of uniformly at random perturbation with ratio 0.2 and 0.3 respectively.
}
\end{figure}

\end{document}